\documentclass{article}[11pt]

\usepackage{amsmath, amsthm, amssymb}
\usepackage[english]{babel}
\usepackage{dsfont}
\usepackage{mathdots}
\usepackage{natbib}
\usepackage{graphicx}
\usepackage{enumerate}
\usepackage{color}
\usepackage{float}
\usepackage{placeins}
\usepackage{url}
\usepackage{comment}
\usepackage{todonotes}
\usepackage{rotating}
\usepackage{multirow}
\usepackage{diagbox}
\usepackage[a4paper, left=2.5cm, right=2.5cm, top=3cm, bottom=3cm]{geometry}

\newtheorem{thm}{{Theorem}}

\newtheorem{cor}[thm]{{Corollary}}

\newcommand{\var}{\mbox{Var}}

\newcommand{\BM}{{(La)}}
\newcommand{\KC}{{(RoLa)}}
\newcommand{\YW}{{(YW)}}
\newcommand{\CL}{{(Dantzig)}}
\newcommand{\TC}{{(TD)}}

\newcommand{\st}{{(st)}}
\newcommand{\ESTVECLASSO}{\emph{Vec-Lasso}}
\newcommand{\ESTROWLASSO}{\emph{Row-Lasso}}
\newcommand{\ESTROWDANTZIG}{\emph{Row-Dantzig}}

\def\eps{\varepsilon}
\newcommand{\gp}{g(p,d,n,\tau)}
\newcommand{\gps}{g(p,d,s,n,\tau)}

\newcommand{\textminbf}[1]{{\scriptsize{[\textit{#1}]}}}

\newcommand{\R}{\mathds{R}}
\newcommand{\C}{\mathds{C}}

\newcommand{\Z}{\mathds{Z}}
\newcommand{\E}{\mathds{E}}
\newcommand{\argmin}[1]{\operatornamewithlimits{argmin}_{#1}}
\newcommand{\veco}{\operatorname{vec}}
\newcommand{\diag}{\operatorname{diag}}
\newcommand{\THRl}{\operatorname{THR}_{\lambdan}}

\newcommand{\ind}{\mathds{1}}
\newcommand{\Gammah}{\hat{ {\Gamma}}^{(st)}(0)}

\newcommand{\Gammas}{{\Gamma}^{(st)}(0)}
\newcommand{\Sigmah}{\hat \Sigma_\eps}

\newcommand{\A}{\mathds{A}}
\newcommand{\lambdan}{\lambda_n}
\newcommand{\sign}{\operatorname{sign}}
\newcommand{\Sigmaeps}{\Sigma_\eps}
\newcommand{\fracd}[2]{#1/#2}

\allowdisplaybreaks
\begin{document}

\begin{center}
{\bf  \huge Statistical Estimation of High-Dimensional Vector Autoregressive Models} 
\end{center} 
\begin{center}
{\bf Jonas Krampe$^{1}$ and Efstathios  Paparoditis$^{2}$}
\end{center}  
\noindent $^1$ University of Mannheim; j.krampe@uni-mannheim.de \\ $^2$ University of Cyprus; stathisp@ucy.ac.cy

\begin{abstract}
High-dimensional vector autoregressive (VAR) models are important tools for  the analysis of multivariate time series.  This paper  focuses on high-dimensional time series  and on the different regularized estimation procedures proposed for fitting sparse  VAR models to such  time series.   Attention is paid to the different sparsity assumptions imposed on the VAR parameters  and how these sparsity assumptions  are related  to the particular consistency 
properties of the estimators established. A sparsity scheme for  high-dimensional VAR models is proposed which is found to be  more appropriate for the  time series setting considered.  Furthermore,  it is shown that, under this  sparsity setting,   threholding  extents the consistency properties of  regularized estimators  to a wide range 
of matrix norms. Among other things, this enables application of the VAR parameters estimators to different  inference problems, like  forecasting or  estimating the second-order characteristics of the underlying  VAR process.  Extensive simulations compare  the finite sample behavior  of  the different regularized estimators  proposed using a variety  of performance criteria. 
\end{abstract}

\vspace{0.2cm}

\section{Introduction}
The vector autoregressive (VAR) model is one of the most prominent and frequently  used models for analyzing multivariate time series;  see among others the  textbooks by  \cite{BrockwellDavis1991}, \cite{reinsel2003elements}, \cite{luetkepohl2007new}, \cite{tsay2013multivariate} and  \cite{kilian2017structural}. Due to the increasing availability of time series data, high-dimensional VAR models has attracted the interest of many researchers during the last two decades. Initiated by developments in the {i.i.d.} setup,  statistical methods have been proposed 
for inferring properties of  high-dimensional VAR models.
However, to make statistical inference for such models  possible, 
the  model complexity has to be reduced. 
One way to achieve this, is to  limit the allowed direct  influences between the variables (time series) involved in the high-dimensional VAR system. A common  strategy toward this,  is to impose  some kind of sparsity or of approximately sparsity assumptions.
For  sparse VAR models, the dimension of the unknown parameters is considerably  reduced by assuming that a large number of these  parameters  is zero. That is,  only  few ``variables'' out of a large set of potential ``explanatory variables'' are allowed  to have a direct influence on  the other variables of the system. In an approximately sparse setting, it is allowed that a large number of parameters is not exactly 
 zero but rather small in magnitude. As we will see, most of the   sparsity patterns used in the literature are motivated by and related to 
 the  particular estimation method used. Apart from reviewing the different sparsity settings used in the literature and the corresponding estimation methods proposed, one of the aims  of this paper is to  introduce  a sparsity pattern  for high-dimensional VAR models  which is more appropriate  for  VAR  models. Furthermore, for the introduced sparsity setting, we 
 use thresholding of  regularized estimators as a tool to  extent their consistency  to a wide range 
 of matrix norms.   

 The main estimation strategy followed in the high-dimensional setting is to use some kind of regularized estimator. In this  context, three different procedures  have been proposed: Regularized least squares (LASSO), see  \cite{basu2015,kock2015oracle}; regularized maximum-likelihood estimators, see among others \cite{basu2015,davis2016sparse}, and regularized Yule-Walker estimators using the CLIME approach or the Dantzig estimator, respectively; see \cite{han2015direct,wu2016performance}. We refer to   \cite{cai2011constrained} for the CLIME method, and to \cite{candes2007dantzig} for the Dantzig estimator.  Different sparsity patterns are used in the aforementioned   papers and 
 an overview of these different patterns  as well as on their impact on  the consistency properties of the  estimators used,  is 
 given in the next section.

We first fix some notation. For a vector $x\in \R^d$,  $\|x\|_0 = \sum_{j=1}^d \ind(x_j \not = 0)$ denotes the number of non-zero coefficients, $ \| x \|_1 = \sum_{j=1}^d |x_j|$, $\| x \|_2^2 = \sum_{j=1}^d |x_j|^2$, and $\|x\|_\infty=\max_j |x_j|$ denote the $ l_1$, $l_2$ and  $l_\infty$ norm, respectively. Furthermore, for an $r\times s$ matrix $B=(b_{i,j})_{i=1,\ldots,r, j=1,\ldots,s}$, $\|B\|_l=\max_{x \in \R^s : \|x\|_l=1} \|Bx\|_l, l \in [1,\infty]$ with $\|B\|_1=\max_{1\leq j\leq s}\sum_{i=1}^r|b_{i,j}|=\max_j \| B e_j\|_1$, 
$\|B\|_\infty=\max_{1\leq i\leq r}\sum_{j=1}^s|b_{i,j}|=\max_{i} \| e_i^\top B\|_1$, and  $\|B\|_{\max}=\max_{i,j} |e_i^\top B e_j|$, where $e_j=(0,\ldots,0,1,0,\ldots, 0)^\top$ denotes  a  vector of appropriate dimension with the one appearing in the $j$th position and zero elsewhere.  Denote the largest absolute eigenvalue of a square matrix $B$ by $\rho(B)$ and note that $\|B\|_2^2=\rho(BB^\top)$. The $d$-dimensional identity matrix is denoted by $I_d$. Furthermore,  for the Kronecker product of  two matrices $A$ and $B$ we write  $A\otimes B$, and for a matrix $A$, $\veco(A)$ denotes its vectorization  obtained  by stacking the columns to a vector; see among others Appendix A.11 in \cite{luetkepohl2007new}. For a vector $x\in \R^d$, $\diag(x)$ is a diagonal matrix of size $d$ the  entries of which on the main diagonal are given by the elements of the vector $x$.

Let $\{X_t,t \in \Z\}$ be a $d$-dimensional vector autoregressive process of order $p$, in short VAR$(p)$, given by 
\begin{align}
    X_t=\sum_{j=1}^p A_j X_{t-j}+\eps_t, \label{eq.VARp}
\end{align}
where $\eps_t$ is a white noise process  with covariance $\Sigmaeps$. We first summarize  some results for VAR$(p)$ processes which will be  used later on and which can be found in many  standard textbooks for time series analysis; see among others,  \cite{luetkepohl2007new} and \cite{,tsay2013multivariate}. A VAR$(p)$ model can be stacked to a VAR$(1)$ model as  follows. Let
$W_t = (X_t^\top,X_{t-1}^\top,\dots,X_{t-p+1}^\top)^\top \in \R^{dp\times 1}$, 
$$\mathds{A}= \begin{pmatrix}
A_1& A_2 & \dots & A_p\\
I_d & 0 & \dots & 0\\
0 & \ddots & \ddots & \vdots\\
0 & \dots & I_d & 0 
\end{pmatrix}\in \R^{dp\times dp},
\E=e_1\otimes I_d=
\begin{pmatrix}
I_d \\
0\\
\vdots\\
0
\end{pmatrix}\in \R^{dp\times d} \text{ and } U_t=\E
\eps_t.
$$
Then,  
$$ W_t= \mathds{A} W_{t-1} + U_t $$ 
is the VAR$(1)$ stacked form of (\ref{eq.VARp}) and  $X_t=\E^\top W_t$. 
It is well known that a VAR$(p)$ process is stable if $\det(I-\sum_{s=1}^p A_s z^s)\neq 0$ for all $|z|\leq 1$ or equivalently, if $\rho(\A)<1$. In what follows, we 
assume that this stability condition  is satisfied. A stable VAR$(p)$ process has the moving average representation $X_t=\E^\top \sum_{j=0}^\infty \A^j U_{t-j}$ and its 
autocovariance matrix function  $ \Gamma: \Z \rightarrow \R^{d\times d}$, can be expressed 
as
\begin{equation}
    \Gamma(h)=\left\{ \begin{array}{lll}
    \E^\top \A^h \sum_{j=0}^\infty \A^j \E \Sigmaeps \E^\top (\A^\top)^j \E & & \mbox{for} \  h \geq 0,\\
    & & \\
    \Gamma(-h)^\top & & \mbox{for} \ h <0.
    \end{array} \right. \label{eq.ACF}
\end{equation}
The spectral density matrix of the VAR(p) process at frequency $\omega \in [-\pi,\pi]$ is given by 
\begin{align}
    f(\omega)=\frac{1}{2\pi}\mathcal{A}^{-1}(\exp(-i\omega)) \Sigmaeps  \big(\mathcal{A}^{-1}(\exp(i\omega))\big)^\top,
    \label{eq.spec}
\end{align}
and its inverse by 
\begin{align}
    f^{-1}(\omega)= 2\pi \mathcal{A}(\exp(i\omega))^\top \Sigmaeps^{-1} \mathcal{A}(\exp(-i\omega)), \label{eq.spec.inv}
\end{align}
where $\mathcal{A}(z)=I_d - \sum_{s=1}^p A_s z^s$, $z\in \C$. 

Given observation $X_1,\dots,X_n$, we can write \eqref{eq.VARp} in a compact form as the  regression equation,    
\begin{align}
    \mathcal{Y}=\mathcal{X} B+ \mathcal{E}, \label{eq.VARregression}
\end{align}
where $\mathcal{Y}=(X_n,\dots,X_{p+1})^\top$ and $\mathcal{E}=(\eps_n,\dots,\eps_{p+1})^\top$ are $(n-p)\times d$ dimensional matrices, $\mathcal{X}=(W_{n-1},\dots,W_p)^\top$ is $(n-p)\times dp$ dimensional matrix and $B=(A_1^\top,\dots,A_p^\top)\in \R^{dp\times d}$.

Let $\widehat A_1,\dots,\widehat A_p$ be an estimator of the VAR parameter matrices $ A_1,\dots, A_p$ and let $\hat \A$ be the corresponding stacked matrix version of these estimators. The innovations can be estimated by $\hat \eps_t=X_t-\sum_{s=1}^p \hat A_s X_{t-s}$, $ t=p+1,\dots,n$, and can be used to construct estimators for the 
covariance matrix $ \Sigma_\varepsilon$. Successful  applications of the estimated VAR model require  consistency of the  estimators  of $ A_1, \ldots, A_p$ and $ \Sigma_\varepsilon$ used. Since for a fixed dimension $d$, the commonly used  matrix norms are  equivalent, it is  not important  in a  low-dimensional setting, with respect to which matrix norm  consistency of the estimators is established. This however, changes in the high-dimensional setting. As we will see in Section~\ref{sec.2}, some of the estimators proposed in the literature,  are consistent with respect to some matrix norms only. Notice that  the consistency requirements on the estimators $ \widehat{A}_s$ and  $ \widehat{\Sigma}_\varepsilon$,  also depend on the applications of the VAR model one has in mind. For instance, for consistency of  the one step ahead forecast  $\hat X_{n+1}=\sum_{s=1}^d \hat A_s X_{n+1-s}$, 
it suffices to have consistency of the estimators $ \widehat{A}_s$  with respect to the $\|\cdot\|_\infty$ norm, that is  $\sum_{s=1}^d \|A_s-\hat A_s\|_\infty=\|\A-\hat \A\|_\infty=o_P(1)$. However,
if interest is directed towards estimating the autocovariance matrix function $ \Gamma(h)$, then, additionally,  
$\|\A-\hat \A\|_1=\|\A^\top-\hat \A^\top\|_\infty=o_P(1)$ is required; see also (\ref{eq.ACF}). If one wants to 
to consistently estimate the  spectral density matrix $f$ of the model, then additionally consistency with respect to the $\|\cdot\|_1$ is required, i.e., $\sum_{s=1}^d \|A_s-\hat A_s\|_1=o_P(1)$. Interest in consistently estimating
the aforementioned second-order characteristics  of the VAR model arises, for instance, 
in the context of 
bootstrap-based inference for high-dimensional VAR models; see \cite{krampe2018bootstrap}. 
Observe that   $\|\A-\hat \A\|_\infty=o_P(1)$ and $\|\A-\hat \A\|_1=o_P(1)$ implies that  $\|\A-\hat \A\|_l=o_P(1)$ for any  $l \in [1,\infty]$.

\section{Different Sparsity Patterns for VAR Models} \label{sec.2}
In this section we  discuss in more detail 
some of the sparsity patterns that have been  used in the literature.
We first note that 
there is a strong connection between the particular sparsity 
assumptions made and the  error bounds obtained for the corresponding parameter estimators. 
Moreover, the sparsity patterns imposed are often motivated by the particular estimation procedure used. Hence, and 
for a  better comparison of the different sparsity patterns used, we also include in our 
discussion  the estimators developed and the error bounds obtained. We focus in the following on   $\ell_1$ penalized estimators  with  tuning parameters always  denoted by $\lambdan$. 

We begin with the sparsity pattern used in  \cite{basu2015}. The authors  use a  vectorized version of \eqref{eq.VARregression}, i.e., $\veco(\mathcal{Y})=(I_d\otimes\mathcal{X})\veco(B)+\veco(\mathcal{E})$ and  formulate the following $\ell_1$-penalized estimator 
\begin{align} \label{eq.mse.basu}
    \hat \beta^{\BM}=\argmin{\beta\in R^{d^2p}} \frac{1}{n-p} (\veco(\mathcal{Y})-(I_d\otimes\mathcal{X})\beta)^\top W(\veco(\mathcal{Y})-(I_d\otimes\mathcal{X})\beta)+\lambdan \|\beta\|_1,
\end{align}
where $W$ is a weighting matrix. $W=I_{d(n-p)}$ leads to a $\ell_1$-penalized least squares estimator and $W=(\Sigmaeps^{-1}\otimes I_{n-p})$  to a $\ell_1$-penalized maximum likelihood estimator.
Weighting is helpful if $\Sigmaeps$ is not well  approximated by a diagonal matrix $\sigma^2 I_d$ for some $\sigma^2>0$.
Note \cite{basu2015} considered Gaussian innovations. Furthermore, they assumed that $\veco(B)$ is a sparse vector in the sense that  $\|\veco(B)\|_0=k$.
Recall that the parameter matrices $A_1,\dots,A_p$ have  $d^2p$ unknown coefficients. Hence, by this sparsity assumption, the number of non-zero coefficients of the VAR system   is limited  by $k$. For the estimator (\ref{eq.mse.basu}), \cite{basu2015} obtained on a set having  high probability, the  error bound 
\begin{align}
    \|\hat \beta^{\BM}-\veco(B)\|_1\leq C k  \sqrt{\log (pd^2)/(n-p)},
\end{align} 
where $C$ is some constant depending on properties of the process $\{X_t\}$ but not on $n,p$ and $k$. Since  $\|\veco(\cdot)\|_1$ is the sum of all component-wise absolute errors, an error bound for $\|\veco(\cdot)\|_1$ implies an error bound with respect to  most matrix norms. Let $\hat A_1^{\BM},\dots,\hat A_p^{\BM}$ be the estimators
of the parameter matrices corresponding to $\hat \beta^{\BM}$. We  have  $\sum_{s=1}^p \|\hat A_s^{\BM}-A_s\|_1\leq \|\hat \beta^{\BM}-\veco(B)\|_1$ and $\sum_{s=1}^p \|\hat A_s^{\BM}-A_s\|_\infty\leq \|\hat \beta^{\BM}-\veco(B)\|_1$. 
Limiting the total number of non-zero coefficients by $k$, has, however,  a major impact on the growth rate allowed for the dimension $d$ of the VAR system. To elaborate,  consider the case $p=1$ and assume  that   
the VAR(1)  process solely  consists of  $d$ univariate AR$(1)$ processes, i.e., that $A$ is a diagonal matrix. Then, without further sparsity restrictions we have $\|\veco(A)\|_0=d$. Thus, and  by  ignoring the log-term,  the estimator $\hat A^{\BM}$ of \cite{basu2015} is  consistent if the dimension $ d$ increases slower than $\sqrt{n}$, that is if $d=o_P(\sqrt{n})$. This implies that the sparsity pattern $\|\veco(A)\|_0\leq k$ used in \cite{basu2015}, can be satisfied   in the high-dimensional setting considered, if the majority of the time series included in the VAR(1) system are essentially  white noise series. 
This seems, however,   to be  rather restrictive   for time series.   This is so since  it is not uncommon to assume that some kind of interactions  between the components of the $d$-dimensional process $\{X_t\}$ exists. If this is the case and the time series included in the VAR(1) system are not white noises,   then  $\|\veco(A)\|_0\geq d$ which contradicts  the  sparsity condition $\|\veco(A)\|_0 = k$. Therefore,  a price for obtaining consistency results with respect to the strong    $\|\veco(\cdot)\|_1$  norm seems to be paid by  the rather restrictive sparsity assumptions one has to impose on the parameters of the VAR system.

Instead of using a vectorized version of \eqref{eq.VARregression}, \cite{kock2015oracle} partitioned the same regression equation 
into single equations by formulating  one equation 
for each time series, i.e., $\mathcal{Y}e_j=\mathcal{X} B e_j+ \mathcal{E}e_j$ where $ j=1,\dots,d$. They then propose  the following $\ell_1$-penalized least squares estimator of $ \beta_j =Be_j$, 
\begin{align}
    \hat \beta_j^{\KC}=\argmin{\beta \in \R^{dp}} \frac{1}{n-p}\|\mathcal{Y}e_j- \mathcal{X} \beta\|_2+\lambdan\|\beta\|_1 \label{eq.row.LASSO}.
\end{align}
Notice that the tuning parameter $\lambdan$ may differ from equation to equation, i.e., $\lambda_n$ may depend on  $j\in \{1,2, \ldots, d\}$. In  terms of the original VAR$(p)$ equation \eqref{eq.VARp}, it is clear that this estimator is formulated row-wise, that is, each time series is modelled  individually by fitting  a $l_1$-regularized regression. This has the drawback that a weighting as in \eqref{eq.mse.basu} can not be  easily implemented.  However, one  advantage is that, now,  the sparsity assumptions can be formulated row-wise, that is $e_j^\top(A_1,\dots,A_p)$ can be assumed to be  a sparse vector. \cite{kock2015oracle} assume  that  $\sum_{s=1}^d \|e_j^\top A_s\|_0=\|B e_j\|_0\leq k_j$. Under the assumption of  Gaussian innovations, they  obtain on a set with high probability, that
\begin{align}
    \|\hat \beta_j^\KC-Be_j\|_1\leq C k_j \sqrt{\log(pd)/(n-p)}, \label{eq.error.KC}
\end{align}
where is the above bound 
some  additional log-terms have been omitted for simplicity. 
Let $\hat A_1^\KC,\dots,\hat A_p^\KC$ be the estimators of the  matrices $ A_s$ corresponding to $(\hat \beta_1,\dots,\hat \beta_d)$. Then, the bound in  \eqref{eq.error.KC} expressed with respect  to the rows of  $ \widehat{A}_s-A_s$,  implies that  $\sum_{s=1}^p \|\hat A_s^\KC-A_s\|_\infty\leq C \max_j k_j \sqrt{\log(pd)/(n-p)}$. Hence, and according to this approach,  the  parameters of the VAR system can  consistently be estimated with respect to the $\|\cdot\|_\infty$   matrix norm. Furthermore, the corresponding sparsity pattern only requires  that the number of non-zero coefficients within the $j$th  row is limited by $\max_j k_j$. Thus, each time series at time point $t$ can  be directly affected by $\max_j k_j$ other lagged variables. Therefore, this sparsity pattern is,  more flexible than the one considered in Basu and Michailidis (2015)  in which  the total number of non-zero coefficients  is limited, i.e.,   $\|\veco(B)\|_0=k$ is assumed. To further clarify the differences, consider again the example of a VAR$(1)$ process which solely consists of $d$ univariate AR$(1)$ processes. Then, $\max_j \|e_j^\top A\|_0=1$ and $\hat A^\KC$ is consistent if $d=o_P(\exp(n))$. Hence, the dimension $d$ of the system can grow much faster compared to what is allowed for  the sparsity pattern used in \cite{basu2015} and for which, as we have seen, $d=o_P(\sqrt{n})$. 

However, since  the estimators $\hat A_1^\KC,\dots,\hat A_p^\KC$ are constructed  row-wise by fitting a regularized regression to each time series, an error bound with respect to the  $\|\cdot\|_1$ norm,  cannot easily be obtained.  It may even not be possible without imposing further assumptions on the VAR process. Furthermore,  within  this sparsity framework, it is possible that there exists a time series which affects all others, that is,  for some $j$ the coefficient  matrix $A_j$ has a dense column. Regarding  the bound with respect to $ \|\veco(\cdot)\|_1$ of the estimator (\ref{eq.row.LASSO}),  we have  \begin{align*}
\|\veco((\hat A_1^\KC,\dots,\hat A_p^\KC))-\veco(B)\|_1 & =\sum_{j=1}^d \|\hat \beta_j^\KC-\beta_j\|_1 \leq C \sum_{j=1}^d k_j \sqrt{\log(pd)/(n-p)},
\end{align*}
which leads to the following bound with respect to the  matrix norm $\|\cdot\|_1$:
\[ \|\hat \A^\KC-\A\|_1\leq\sum_{s=1}^p \|\hat A_s^\KC-A_s\|_1\leq C \sum_{j=1}^d k_j \sqrt{\log(pd)/(n-p)}.\] 
It is not clear if this bound can be improved 
to  $Cd\sqrt{\log(pd)/n}$.
Moreover,  the considered  sparsity pattern, which is more flexible than the one in Basu and Michailidis (2015), is  possible only if consistency of the estimators with respect to the $\|\cdot\|_\infty$ norm  is required. 
This implies that  this estimator can be used in a   high-dimensional setting  for forecasting purposes but it may be of limited value if one is interested in   estimating the second-order properties of the VAR model, like the autocovariance matrix $ \Gamma(h)$ or the spectral density matrix $ f(\lambda)$.

The approaches of \cite{basu2015} and \cite{kock2015oracle} are inspired by the i.i.d regression setup. In contrast to this, the approach of \cite{han2015direct} is inspired by the setup of    high-dimensional covariance estimation. In this setup, CLIME (constrained $\ell_1$-minimization for inverse matrix estimation), see \cite{cai2011constrained}), provides  an approach to estimate the inverse covariance matrix and it is based on the Dantzig estimator, see \cite{candes2007dantzig}. The corresponding estimator of the precision matrix $ \Sigma^{-1}_\varepsilon$ is obtained  as the solution of the following optimization problem,
\begin{align}
    \min_{\Omega \in \R^{d\times d}} \sum_{i,j=1}^d |e_i^\top \Omega e_j| \text{ s.t. } \| \Sigma_{\eps,n} \Omega-I_d\|_\infty \leq \lambdan,  \label{eq.CLIME}
\end{align}
where $\Sigma_{\eps,n}$ is  the sample covariance matrix and $\lambda_n$ a tuning parameter. The above optimization problem can be splited   into sub-problems, that is,
$\hat \beta_j=\argmin{\beta \in \R^d} \|\beta\|_1$  s.t. $\|\Sigma_{\eps,n} \beta-e_j\|_\infty\leq \lambdan$. 
This sub-problem strategy enables  the derivation of   error bounds with respect to the  $\|\cdot\|_1$ norm  without the need for any  additional thresholding of the estimators obtained; see \cite{cai2011constrained} for details. \cite{han2015direct} focus on VAR$(1)$ model and  use the Yule-Walker equation $\Gamma(-1)=\Gamma(0)A^\top$ to formulate an optimization problem similar to \eqref{eq.CLIME}. They derive the following estimator of $ \beta_j = A^\top e_j$, $j\in \{1,2, \ldots, d\}$, 
\begin{align} \label{eq.11}
    \hat \beta_j^\YW = \argmin{\beta \in \R^{d}} \|\beta\|_1 \text{ s.t. } \|S_0 \beta-S_1 e_j\|_{\max}\leq \lambdan,
\end{align}
where $S_0=1/n\sum_{t=1}^n X_t X_t^\top$ and  $S_1=1/(n-1)\sum_{t=1}^{n-1} X_t X_{t+1}^\top$ are  sample autocovariances at lag zero and lag minus one,  respectively. 
Let $(\hat A^\YW)^\top=(\hat \beta_1^\YW : \dots : \hat \beta_d^\YW)$ be the estimator of $A^\top$  in matrix form.   \cite{han2015direct} follow \cite{cai2011constrained} regarding the sparsity assumptions they impose on the VAR(1) system. In particular, they assume that 
\begin{align}
    A\in  \mathcal{M}(q,s,M)=\Big\{B\in \R^{d\times d} : \max_{1\leq j \leq d} \sum_{i=1}^d | B_{j,i}|^q \leq s,  \|B\|_\infty\leq M\Big\}. \label{eq.row-wise.approx.sparse}
\end{align} 
This means, the rows of $A$, i.e., the columns of $A^\top$, are  considered as approximately sparse and bounded in $\ell_1$-norm by the positive constant $M$. Under  the sparsity  assumption (\ref{eq.row-wise.approx.sparse}), they obtain for  Gaussian innovations the following error bound, on a set with high probability,
\begin{align}
    \|A^\top-(\hat A^\YW)^\top\|_1=\|A-\hat A^\YW\|_\infty\leq C M \|\Gamma(0)^{-1}\|_1 s \sqrt{\log(d)/n}. \label{eq.YW}
\end{align}
As mentioned, the norm $\|\cdot\|_1$ arises canonically for the Dantzig estimator since the optimization sub-problems are build up column-wise for  the matrix $A^\top$. However, an error bound for $ \|A^\top-(\hat A^\YW)^\top\|_\infty= \|A-\hat A^\YW\|_1$ cannot be derived without further assumptions. In fact,  only the following naive bounds hold
true:
$\|A-\hat A^\YW\|_\infty\leq C M \|\Gamma(0)^{-1}\|_1 d \sqrt{\log(d)/n}$. \cite{wu2016performance} extended the approach of  using the CLIME method for VAR parameter estimation to general VAR$(p)$ processes and to possible non-Gaussian innovations. They focus on  error bounds with respect to the  $\|\cdot\|_{\max}$ norm. For this, they do not need to specify a particular sparsity pattern. Using the same  approximately  sparsity setting
\eqref{eq.row-wise.approx.sparse}, \cite{masini2019regularized} showed that the row-wise Lasso \eqref{eq.row.LASSO} possesses with high-probability and under some restrictions on the growth rates of $n$, $q$, and $s$,  the following  error bound with respect to the $ \|\cdot\|_2$ norm, 
\begin{align}
    \|\hat \beta_j^\KC-\beta_j\|_2^2 \leq C_\tau s \|\Gammas^{-1}\|_2^{(2-q)} \Big[( d^2 p n )^{2/r}/\sqrt{n} \Big]^{2-q}. \label{eq.error.lasso.approximate.sparse}
\end{align}
Here $\tau$ denotes  the number of finite moments of the innovations $ \eps_t$, i.e., $\max_{\|v\|_1\leq 1} (E |v^\top \eps_1\|^\tau)^{1/\tau}\leq c_\tau< \infty$, $ \tau > 4$, and $C_\tau$ denotes a particular  constant depending on $c_\tau$ and on $\tau$ only. For sub-Gaussian innovations, they obtained sharper error bounds which allow for a larger value of the dimension $d$ and which are similar to the rates given in \eqref{eq.error.KC}. 


\section{A Sparsity Setting for  VAR Time Series}
Our aim in this section is twofold. First, and based on the discussion of the previous section, we   introduce a sparsity setting for VAR models which is appropriate for the high-dimensional time series setup considered in this paper. Second, for the sparsity setting introduced, we then   derive 
estimators of the VAR model parameters, which are consistent with respect to all  matrix norms $ \|\cdot\|_l$, for $ l\in [1,\infty]$.

As already mentioned, the aim of any sparsity pattern is  to reduce the complexity of the model such that consistent  estimation becomes possible even  in a high-dimensional setup. Towards developing a sparsity setting which is appropriate for high-dimensional time series,  and in particular for VAR models, it is worth to first recall the meaning of the coefficients of the parameter matrices $ A_s$, $s=1,2, \ldots, p$.  The coefficients $e_j^\top(A_1,\dots,A_p)$, i.e., those in the  $j$th row of the autoregressive matrices,  describe the direct linear influence of all time series (in lagged form), that is of $X_{t-1},\dots,X_{t-p}$, onto the $j$th component at time $t$, that is onto $X_{t;j}$. Furthermore, the coefficients on the  $j$th column of the matrix $A_k$, i.e., the coefficients  $A_k e_j$, describe the direct linear influence of the $j$th component at lag $k$, that is of  $X_{t-k;j}$, onto all time series of the system at time $t$, that is onto $X_t$. Imposing a sparsity pattern on the coefficient matrices  $(A_1,\dots,A_p)$ means that the direct influences between  the different time series are restricted. A reasonable  sparsity pattern will be one in which it is  assumed  that a single time series (including all its past  values up to lag $p$) can affect directly and can be directly affected  only by   a limited number of other time series (including their  lagged versions). This requirement leads to the need of  imposing   sparsity assumptions in the rows and in the columns of 
the matrices $ A_s$, $s=1,2, \ldots,p$. Row-wise sparsity in the form   $\max_j \sum_{k=1}^p\|e_j^\top A_k\|_0\leq s$, means that a time series can be influenced directly only by $s$ other time series (including their lagged values). However, for  column-wise sparsity,  two reasonable  options exist. The first  is the column-wise analog to the aforementined row-wise sparsity, which leads to the requirement  that $\max_j \sum_{k=1}^p\| A_k e_j\|_0\leq s$. This  means that a single time series $j$  and  all its past values, e.g.,  $X_{t-k;j},k=1,\dots,p$, has at most $s$ direct 
channels to affect the elements of the vector $X_t$. The second option one has is  the requirement $\max_j \max_{1\leq k \leq p} \| A_k e_j\|_0\leq s$. This  means that a single time series in one of its lagged versions, for instance  $X_{t-p;j}$, can affect at most $s$ other time series, that is at most  $s$ of the components of the vector  $X_t$. Hence, in the second option a single time series with all its lagged values  has at most $s\times p$ channels  to affect directly the components of  $X_t$. We mention here that  if there is a (near to) full interaction among the time series of the system, then it may be  more reasonable to consider alternative approaches for inferring properties of  high-dimensional time series.  Factor models and more specifically, dynamic factor models could   be a possible   alternative   in such a  case. Such an approach   will avoid the  imposition of   (unrealistic)  sparsity assumptions on the interaction between the  time series considered. However, other assumptions  are required in this case, like for instance, that  the evolution of the entire high dimensional system of  time series is driven by few non observable components.  We refer here to  the  surveys \cite{stock2005implications,stock2011dynamic,stock2016dynamic,bai2008large}. See also Chapter 16 in \cite{kilian2017structural}.

The above discussion regarding   an appropriate  sparsity pattern for high-dimensional VAR models, was devoted to the case of the so called  strict sparsity. This is the case where  coefficients are counted only if they are  different from zero.  Nevertheless we may also consider the case of  so called  approximately sparsity.  Towards this,  we adopt the approximately sparsity settings  used for high-dimensional covariance matrices; see among others \cite{bickel2008,rothman2009generalized}. Since  in contrast to covariance matrices, the parameter matrices $A_1,\dots,A_p$ are in general not symmetric,
we state  the following two classes of approximately sparse VAR$(p)$ matrices, where  each one of them refers to the two different column-wise sparsity options we have discussed before. 

\begin{align}
    \mathcal{M}^{(1)}(q,s,M,p)=\Big\{(M_1,\dots,M_p), M_i\in \R^{d\times d} :& \max_{1\leq j \leq d} \max_{1\leq k \leq p} \sum_{i=1}^d | A_{k;i,j}|^q \leq s, \max_{1\leq k \leq p} \|A_k\|_1\leq M,  \nonumber\\
    &
    \max_{1\leq i \leq d} \sum_{k=1}^p \sum_{j=1}^d | A_{k;i,j}|^q \leq s,  \sum_{k=1}^p \| A_k\|_\infty\leq M\Big\}, \label{eq.sparse.pattern}
\end{align}

\begin{align}
    \mathcal{M}^{(2)}(q,s,M,p)=\Big\{(M_1,\dots,M_p), M_i\in \R^{d\times d} :& \max_{1\leq j \leq d} \sum_{k=1}^p \sum_{i=1}^d | A_{k;i,j}|^q \leq s, \sum_{k=1}^p \|A_k\|_1\leq M,  \nonumber\\
    &
    \max_{1\leq i \leq d} \sum_{k=1}^p \sum_{j=1}^d | A_{k;i,j}|^q \leq s,  \sum_{k=1}^p \| A_k\|_\infty\leq M\Big\}, \label{eq.sparse.pattern2}
\end{align}
where $q \in [0,1)$. Notice that  $q=0$ refers to the case of strict sparsity whereas $q>0$ to that of approximately sparsity. In the remaining of this paper,  we focus on the pattern  $\mathcal{M}(q,s,M,p)=\mathcal{M}^{(1)}(q,s,M,p)$ only. This sparsity pattern is a generalization of the one  used in \cite{krampe2018bootstrap} and a subset of the sparsity pattern used in \cite{han2015direct} and \cite{masini2019regularized}. However, \cite{han2015direct} obtained  consistency only with respect to the   $\|\cdot\|_\infty$ norm, i.e., $\|\A-\hat \A^\YW\|_\infty$, whereas our aim is to obtain  consistency with respect  to  $\|\cdot\|_l$   for all values of $l \in [1,\infty]$. As mentioned, this will enable the use of the estimators  obtained in  several   applications, like forecasting or  estimating the second-order structure of the VAR process. 
 Since $\mathcal{M}^{(2)}(q,s,M,p) \subseteq \mathcal{M}(q,s,M,p)$, all results presented here also hold true for the other  sparsity pattern $\mathcal{M}^{(2)}(q,s,M,p)$ stated in (\ref{eq.sparse.pattern2}). If $(A_1,\dots,A_p) \in \mathcal{M}^{(2)}(q,s,M,p)$, then, additionally to $\|\A-\hat \A\|_l=o_P(1)$, $\sum_{k=1}^p \|A_k-\hat A_k\|_l=o_P(1)$, for all  $l \in [1,\infty],$ can be established. This  is important if one wants 
to obtain a consistent estimator of the inverse of the spectral density matrix of the VAR model; see Theorem~\ref{thm.spec} bellow  for details.  Notice  that the two sparsity patterns coincide for VAR$(1)$ models, i.e.,  $\mathcal{M}^{(2)}(q,s,M,1)=\mathcal{M}(q,s,M,1)$.

As we have seen,  regularization is an important tool for  obtaining consistent estimates in a high-dimensional setting. In the context  of  covariance matrix estimation, one approach is  thresholding the sample covariance matrix; see \cite{bickel2008,rothman2009generalized,cai2011adaptive}. Since the sample covariance  matrix is (under certain assumptions) consistent with respect to the   $\|\cdot\|_{\max}$ norm, thresholding helps to transmit  the  component-wise consistency to consistency  with respect to  a matrix norm. The CLIME method, see \cite{cai2011constrained}, achieves   consistency with respect to the $\|\cdot\|_{\max}$  also for the precision matrix, i.e., the inverse of the covariance matrix. As mentioned, \cite{han2015direct} use the CLIME method to estimate the parameter matrix $A$ of a VAR$(1)$ model and they established   $\|A-\hat A \|_{\max}=o_P(1)$. \cite{cai2011constrained} pointed  out that the optimization problem of the CLIME method can be split  into sub-problems which lead to error bounds with respect to  the $\|\cdot\|_1$ norm  without the use  of thresholding. \cite{han2015direct} followed this idea for  constructing an  estimator  of the transposed matrix,  leading to the error bound   \eqref{eq.YW}. \cite{wu2016performance} generalized the approach of \cite{han2015direct} to VAR$(p)$ processes and to possible non-Gaussian time series. In this context,  \cite{wu2016performance} obtained  the result $\|(A_1,\dots,A_p)-(\hat A_1,\dots,\hat A_p)\|_{\max}=o_P(1)$. The same result also can be established  for the Lasso. Hence, in order to obtain a consistent  estimator  for the sparsity pattern \eqref{eq.sparse.pattern} adopted in this paper, we propose to threshold an estimator which fulfills $\|(A_1,\dots,A_p)-(\hat A_1,\dots,\hat A_p)\|_{\max}=o_P(1)$. Toward this, we use the class of thresholding functions given by \cite{cai2011adaptive}. In particular, we require that a thresholding function $\THRl : \R  \to \R$ at threshold level ${\lambdan}$ satisfies the following three conditions:
    \begin{enumerate}
        \item $\THRl(z)\leq c |y|$ for all $z,y$ satisfying  $|z-y|\leq {\lambdan}$ and some $c \in (0,\infty)$.
        \item $\THRl(z)=0$ for $|z|\leq {\lambdan}$.
        \item $|\THRl(z)-z|\leq {\lambdan},$ for all $z\in \R$.
    \end{enumerate}
 For a matrix $A$ with elements $a_{i,j}$, we set $\THRl(A):=(\THRl(a_{i,j}))_{i,j}$, which means that  thresholding is applied component-wise. These conditions are satisfied among others by the soft thresholding, $\THRl^S(z):=\sign(z) (|z|-{\lambdan})_+$,  and the adaptive Lasso thresholding, $\THRl^{al}(z)=z \max(0,1-|{\lambdan}/z|^\nu$, with $\nu\geq1$; see \cite{cai2011adaptive} and Figure 1 in \cite{rothman2009generalized} for an illustration of the different thresholding operations. Note that the hard thresholding, $\THRl^H(z):=z \ind(|z|>{\lambdan})$,  does not fulfill  condition 1) above. 
The following theorem  shows that this thresholding strategy  succeeds in obtaining estimators which are consistent with respect to all matrix norms $ \|\cdot \|_l$, for $ l \in [1,+\infty]$. 

\begin{thm}\label{thm.thres}
Let  $(A_1,\dots,A_p)\in \mathcal{M}(q,s,M,p)$ and assume that $(\hat A_1,\dots,\hat A_p)$ is an estimator which fulfills on a subset $\Omega_n$ of the sample space, the following condition
\begin{align}
  \max_{1\leq s\leq p} \|A_s-\hat A_s\|_{\max}\leq C_1 t_n. \label{eq.ass.thm} 
\end{align}
Then it holds true on the same subset $ \Omega_n$ with thresholding parameter ${\lambdan}=C_1 t_n$, that 
\begin{align}
\|\A-\THRl(\hat \A)\|_l \leq (4+c)C_1^{1-q}s  t_n^{1-q},
\end{align}
for all  $l \in [1,\infty]$. In the above expression,  $c$ is a constant which depends on  the particular  thresholding function  used.
\end{thm}

Since $\|\cdot\|_{\max}\leq \|\cdot\|_2\leq \|\cdot\|_1$, the error bounds given in Section~\ref{sec.2} are 
(not necessarily sharp) error bounds for the element-wise error based on the $ \|\cdot \|_{\max}$ norm. Furthermore, the aforementioned relation  between the matrix norms implies  that  the row-wise Lasso, which is obtained as 
\begin{align}
    \hat \beta_j^{\KC}=\argmin{\beta \in \R^{dp}} 1/(n-p)\|\mathcal{Y}e_j- \mathcal{X} \beta\|_2+\lambdan\|\beta\|_1, \label{eq.row.LASSO2}
\end{align}
as well as the Dantzig estimator for VAR$(p)$ models, that both estimators fulfill the assumptions of Theorem~\ref{thm.thres}. That is,   both estimators can be used to obtain via thresholding, a row- and column-wise consistent estimator of the VAR parameter matrices $ A_s$, $ s=1,2, \ldots, p$. 

For the Lasso estimator, we can use the results of \cite{masini2019regularized}, since the sparsity setting described in \eqref{eq.sparse.pattern} is covered  by the sparsity setting used by these authors. Their results lead to the following error bound for the row-wise lasso, on a set with high probability,
$$
\|\hat \beta_j^{\KC}-\beta_j\|_{\max} \leq C_\tau \|\Gammas^{-1}\|_2^{(2-q)/2} \sqrt{s} \Big(\gp/\sqrt{n}\Big)^{(2-q)/2}, 
$$
where $\gp=( d^2 p n)^{2/\tau}$,
$\tau$ denotes the number of finite moments of the innovations, i.e.,\\ $\max_{\|v\|_1\leq 1} (E |v^\top \eps_1\|^\tau)^{1/\tau}\leq c_\tau\leq \infty, \tau > 4$, and $C_\tau$ denotes a constant depending on $c_\tau$
and $\tau$. Notice that  $\gp=\log(dp)$ in the case of  sub-Gaussian innovations where  all moments exist. In both cases $C_\tau$ depends among others on $\|\Sigmaeps\|_2$. Thus, Theorem~\ref{thm.thres} leads for $l \in [1,\infty]$
to the following bound for the thresholded, row-wise lasso estimator,
$$\|\A-\THRl(\hat \A^\KC)\|_l=O_P\Bigg( \|\Sigmaeps\|_2^{1-q} \|\Gamma(0)^{-1}\|_2^{(2-q)(1-q)/2} s^{1+(1-q)/2}\Big(\gp/\sqrt{n}\Big)^{(2-q)(1-q)/2}\Bigg).$$

For the Dantzig estimator, an error bound with respect to the  $\|\cdot\|_{\max}$ norm can be derived directly and without imposing any sparsity constrains. The Dantzig estimator,    $ \hat B^\CL $,  is given by 
\begin{align}
    \hat B^\CL=\argmin{B\in \R^{dp\times d}} \sum_{j=1}^d \|B e_j\|_1 \text{ s.t. } \|\mathcal{X}^\top\mathcal{X}/(n-p) B-\mathcal{X}^\top \mathcal{Y}/(n-p)\|_{\max}\leq \lambdan, \label{eq.VAR.CLIME}
\end{align}
with $\mathcal{X}$ and $\mathcal{Y}$ defined as  in \eqref{eq.VARregression}.  Notice that (\ref{eq.VAR.CLIME}) or \eqref{eq.VAR.CLIME.comp}, respectively, is a VAR$(p)$ version of the estimator given in (\ref{eq.11}).
\cite{cai2011constrained} pointed out that this optimization problem can be splited  into sub-problems such that parallel-processing can be used to speed up computation. Hence, an estimator  also is  given by $\hat B^\CL=(\hat \beta_1^\CL,\dots,\hat \beta_d^\CL)$, where 
\begin{align}
    \hat \beta_j^\CL=\argmin{\beta\in \R^{dp}}  \|\beta\|_1 \text{ s.t. } \|\mathcal{X}^\top\mathcal{X}/(n-p) \beta-\mathcal{X}^\top \mathcal{Y}e_j /(n-p)\|_{\max}\leq \lambdan, \label{eq.VAR.CLIME.comp}.
\end{align}
$j=1,2, \ldots, d$.
To discuss  the bounds obtained for different estimators, we fix the following notation. Let
\begin{equation} \label{eq.D1}
D_{1,n} = \frac{\sqrt{log(dp)}}{\sqrt{n-p}} + \frac{(dp)^{4/\tau}}{(n-p)^{1-2/\tau}}    
\end{equation}
and
\begin{equation} \label{eq.D2}
D_{2,n} = \frac{\sqrt{log(dp)}}{\sqrt{n-p}} + \frac{(dp)^{1/\tau}}{(n-p)^{1-1/\tau}},    
\end{equation}
where $ \tau > 0 $ is some constant  depending on the moments of the innovations $\varepsilon_t$.

If 
$\{\eps_t\}$ is an {i.i.d.} sequence with $\max_{\|v\|_2\leq 1}(E(v^\top \eps_0)^\tau)^{1/\tau}=:C_{\eps,\tau}<\infty$ for $\tau>2$, then \cite{wu2016performance} showed for the estimator \eqref{eq.VAR.CLIME.comp} the following error bound
\begin{align}
    P\Big(\|Be_j -\hat \beta_j^\CL\|_{\max} &\leq 2 \|\Gammas^{-1}\|_1 (\sum_{j=0}^\infty \|\A^j\|_2 C_{\eps,\tau})^2 \big(D_{1,n} M + D_{2,n}\big)\Big)
    \nonumber \\
    & \geq 1-\frac{dp(n-p)^{1-\tau}}{D^\tau_{2,n}} - dp e^{-C_2^{WW} (n-p) D_{1,n}^2}\nonumber \\
     & \ \ \  -\frac{d^2p^2(n-p)^{1-\tau/2}}{D_{2,n}^{\tau/2}} + d^2p^2e^{-C_1^{WW}(n-p)D^2_{1,n}}\nonumber \\
    & =\tilde p_n^{\CL}. \label{eq.error.bound.wu}
\end{align}
Here 
$C_1^{WW}$ and $C_2^{WW}$ are constants depending on $\tau$ only and $\Gammas=\var((X_p,\dots,X_1))=\var(W_1)$, is the lag zero autocovariance of the stacked VAR$(1)$ model. Notice that the error bound \eqref{eq.error.bound.wu} refers to the case in which the innovations possess only a finite number of moments and a key ingredient in its derivation   is Nageav's inequality,  which \cite{wu2016performance} generalized for dependent sequences of random variables.  The same authors also  obtain  an error bound if all moments of the innovations $\varepsilon_t$  are finite. In this case a sharper bound can be obtained where  polynomial terms do not  occur and the exponential term depends on the tail behavior of the distribution of the innovations. For the sake of an easy presentation, we do not discuss this case here but we will come back 
to it later on. As already mentioned, \cite{wu2016performance} derived this bound without imposing  specific assumptions on   the underlying sparsity setting. For  the sparsity setting used in this section, this bound can be improved. The new  bound 
which is also obtained  using  Theorem~\ref{thm.thres}, is   stated in the following  Corollary~\ref{cor.Dantzig}. This corollary states  that a stable VAR$(p)$ model can be estimated consistently,  with respect to all norms $\|\cdot\|_l, l \in [1,\infty]$,  by  a thresholded Dantzig estimator in a row- and column-wise approximately sparsity setting and with (possibly) non-Gaussian innovations.

\begin{cor} \label{cor.Dantzig}
Let $(A_1,\dots,A_p)\in \mathcal{M}(q,s,M,p)$  and $\{\eps_t\}$ be  an {i.i.d.} sequence with finite $\tau>2$ moments, i.e., $\max_{\|v\|_2=1} (E(v^\top \eps_0)^q)^{1/\tau}=:C_{\eps,\tau}<\infty$. Furthermore, let 
$\lambdan=C (\sum_{j=0}^\infty \|\A^j\|_2 C_{\eps,\tau})^2  D_{2,n} $
for some constant $C>0$, be the tuning parameter for $\hat B^\CL$ and denote its thresholded version by $(\hat A_1^\TC,\dots,\hat A_p^\TC)=\THRl(\hat B^\CL)^\top$. Then it holds true on a set with probability equal or higher to $p_n^\CL$, where 
\begin{align*}
p_n^\CL:=&1- \frac{d^2p^2(n-p)^{1-\tau}}{D_{2,n}^\tau} - d^2p^2e^{-C_2^{WW} (n-p)D_{2,n}^2}-
\frac{d^2p^2(n-p)^{1-\tau/2}}{D_{1,n}^{\tau/2}}-  d^2p^2 e^{-C_1^{WW}(n-p)D_{1,n}^2}, 
\end{align*}
we have that
\begin{align}
\|\Gammas-\mathcal{X}^\top\mathcal{X}/N\|_{\max} \leq  (\sum_{j=0}^\infty \|\A^j\|_2 C_{\eps,\tau})^2 D_{1,n} \leq \lambda_n d^{3/\tau}p^{3/\tau}(n-p)^{1/\tau},
%
\label{eq.Gammas.max.error.bound}
\end{align}
\begin{align}
\|\mathcal{X}^\top\mathcal{E}/N\|_{\max} \leq  (\sum_{j=0}^\infty \|\A^j\|_2 C_{\eps,\tau})^2 D_{2,n}
\leq \lambdan, \label{eq.eps.max.error.bound}
\end{align}
\begin{align}
    \|B-B^\CL\|_{\max}\leq& \|\Gammas^{-1}\|_1 \lambdan\Big( (D^3N)^{1/\tau} (2\|\Gammas^{-1}\|_1\lambdan(1+M (D^3N)^{1/\tau}))^{1-q}\\
    & \ \ \ \ \ \times(1+2^{1-q}+3^{1-q})s+2\Big), \label{eq.Dantzig.max} \\
    =&O_P\Big(\|\Gammas^{-1}\|_1 \lambdan (D^3N)^{1/\tau}\Big),\nonumber 
    \end{align}
if  $ \|\Gammas^{-1}\|_1 \lambdan M (D^3N)^{1/\tau}=o_P(1)$,
where $ D=dp$ and $ N=n-p$. 
Furthermore, using the same notation, we have that it holds true for  all $l\in[1,\infty]$ and on the same set
as above, that
\begin{align}
\|\A-\hat \A^\TC\|_l& \leq (4+c) s \|B-B^\CL\|_{\max}^{1-q} \label{eq.VAR.TC.error.bound}\\
&=(4+c) s\Big(\|\Gammas^{-1}\|_1 \lambdan\Big( (D^3N)^{1/\tau} (2\|\Gammas^{-1}\|_1\lambdan(1+M (D^3N)^{1/\tau}))^{1-q}\nonumber \\
& \ \ \ \ \ \times(1+2^{1-q}+3^{1-q})s+2\Big)\Big)^{1-q}, \nonumber
\end{align}
where $c$ is a constant which  depends on the 
thresholding operation used.
\end{cor}

In some applications of  VAR models, estimation of the  covariance matrix of the innovations is also required. Given some estimators $(\hat A_1,\dots,\hat A_p)$, estimates of the  innovations can be obtained  as $\hat \eps_t=X_t-\sum_{s=1}^p \hat A_s X_{t-s}, t=p+1,\dots,n$. For  simplicity, we omit the centering of the  residuals $\hat \eps_t$, but we recommend  to use it in practise.
To obtain an estimator of the innovations covariance matrix, several  approaches can  be used;
we refer here to   \cite{bickel2008,rothman2009generalized,cai2011adaptive,cai2011constrained,cai2016estimating}. For instance, using the previously mentioned  thresholding functions with  threshold parameter  $\lambdan$, we obtain 
\begin{align}
\Sigmah^{(Thr)}=\THRl(\frac{1}{n-p}\sum_{t=p+1}^n \hat \eps_t \hat \eps_t^\top), \label{eq.eps.est.thres}
\end{align}
while the CLIME estimator of $ \Sigma_\varepsilon^{-1}$  with  tuning parameter $\lambdan$ is given by  
\begin{align}
 \widehat { (\Sigmaeps^{-1})}^{(CLIME)}=(\argmin{\beta \in \R^{d}} \|\beta\|_1 \text{ s.t. } \|\fracd{1}{(n-p)}\sum_{t=p+1}^n \hat \eps_t \hat \eps_t^\top \beta - e_j\|_{\max}\leq \lambdan)_{j=1,\dots,p}. \label{eq.eps.est.clime}
\end{align}
Since the estimated innovations $\hat{\varepsilon}_t$ are used instead of the true  ones, an additional estimation error may occur which  depends on the behavior of the particular  estimators $(\hat A_1,\dots,\hat A_p)$ used. In particular we have 
\begin{align}
    \|&\frac{1}{n-p}\sum_{t=p+1}^n \hat \eps_t \hat \eps_t^\top-\frac{1}{n-p}\sum_{t=p+1}^n \eps_t \eps_t^\top\|_{\max}\leq \|\A-\hat \A\|_1 \Big(2 \|\mathcal{X}^\top\mathcal{E}/(n-p)\|_{\max} \nonumber \\
    &+\|\A-\hat \A\|_1 (\|\Gammas\|_{\max}+\|\Gammas-\mathcal{X}^\top\mathcal{X}/(n-p)\|_{\max})\Big). \label{eq.estimated.eps.error.max}
\end{align}
Corollary~\ref{cor.innovation.clime} bellow gives  the error bound  obtained when the estimator $(\hat A_1^\TC,\dots,\hat A_p^\TC)$ discussed in Corollary~\ref{cor.Dantzig} is used. Notice  that for Gaussian innovations we have $\|\Sigmaeps-\frac{1}{n-p}\sum_{t=p+1}^n \eps_t \eps_t^\top\|_{\max}=O_P(\sqrt{\log(d)/n})$. For non-Gaussian innovations and using the results already presented, we  have on a set with probability of at least $p_n^\CL$ and using \eqref{eq.Gammas.max.error.bound}, that  the bound
\begin{align}
    \|&\Sigmaeps-\frac{1}{n-p}\sum_{t=p+1}^n \eps_t \eps_t^\top\|_{\max}\leq C_{\eps,\tau}^2 \Big( \frac{\sqrt{\log(d)}}{\sqrt {n-p}}+\frac{d^{4/\tau}}{(n-p)^{1-2/\tau}}\Big). \label{eq.eps.error.bound}
\end{align}
Note that $1/(n-p)\sum_{t=p+1}^n \eps_t \eps_t^\top=\mathcal{E}\mathcal{E}^\top/(n-p)$ and $\mathcal{E}$ takes the role of $\mathcal{X}$ for $\A\equiv0$. This means that the fact that  estimated residuals  are used instead of the true innovations,  affects the corresponding convergence rate only if the bound in \eqref{eq.estimated.eps.error.max} is larger  than the bound in \eqref{eq.eps.error.bound}. For $q<1/2$, \eqref{eq.eps.error.bound} is usually larger. It also depends, therefore,  on the underlying sparsity setting;  see Corollary~\ref{cor.innovation.clime} bellow.  

\begin{cor} \label{cor.innovation.clime}
Under the assumptions of Corollary~\ref{cor.Dantzig}, we have on a set with probability of at least $p_n^\CL$, that for $\hat \eps_t=X_t-\sum_{s=1}^p \hat A_s^\TC X_{t-s}, t=p+1,\dots,n$,
\begin{align}
    \|&\frac{1}{n-p}\sum_{t=p+1}^n \hat \eps_t \hat \eps_t^\top-\frac{1}{n-p}\sum_{t=p+1}^n \eps_t \eps_t^\top\|_{\max}\leq 2(4+c)^2 \|\Gammas\|_{\max} s^2 \\
    &\ \ \ \ \ \ \ \ \times \Big(2 \|\Gammas^{-1}\|_1 (\sum_{j=0}^\infty \|\A^j\|_2 C_{\eps,\tau})^2 \Big[D_{1,n} M+D_{2,n}\Big]\Big)^{2(1-q)}.\nonumber
\end{align}
\end{cor}

\begin{thm} \label{thm.eps.est}
Under the assumptions of Corollary~\ref{cor.Dantzig} and if $\Sigmaeps\in \mathcal{M}(q_\eps,s_\eps,M_\eps,1)$, we have on a set with probability of at least $p_n^\CL$, 
\begin{align}
    &\|\THRl(\frac{1}{n-p}\sum_{t=p+1}^n \hat \eps_t \hat \eps_t^\top)-\Sigmaeps\|_l\leq (4+c) s_\eps \Big(C_{\eps,\tau}^{2}
     \Big( \frac{\sqrt{\log(d)}}{\sqrt {n-p}}+\frac{d^{4/\tau}}{(n-p)^{1-2/\tau}}\Big)\nonumber \\
     & \ \ \ \ +2(4+c)^2 \|\Gammas\|_{\max} s^2 
    \Big(2 \|\Gammas^{-1}\|_1 (\sum_{j=0}^\infty \|\A^j\|_2)^2 \Big[
    D_{1,n} M + D_{2,n}
    \Big]\Big)^{2(1-q)}\Big)^{1-q_\eps} \!\!\!\!\!\!\!\!\!\!
    ,
\end{align}
for all $l \in [1,\infty]$, where
$\hat \eps_t=X_t-\sum_{s=1}^p \hat A_s^\TC X_{t-s}, t=p+1,\dots,n$.
\end{thm}
 Theorem~\ref{thm.eps.est} follows directly from Theorem~\ref{thm.thres}, equation  \eqref{eq.eps.error.bound}, and Corollary~\ref{cor.innovation.clime}.

In the remaining of this section,  we propose estimators of the autocovariance function \eqref{eq.ACF} and of the spectral density,  more precisely of the inverse of the spectral density matrix of the VAR model; see  \eqref{eq.spec.inv}. Regarding  the autocovariance function the following expression is useful for our derivations,
\begin{align}
    \Gamma(h)^\st=\left\{ \begin{array}{lll}
    \sum_{j=0}^\infty \A^h \A^j \Sigma_U (\A^j)^\top & & \mbox{for} \ h\geq 0, \\
    & & \\
    (\Gamma(-h)^\st)^\top & & \mbox{for} \ \ h <0, 
    \end{array} \right. \label{eq.Gamma.st}
\end{align}
where $\Sigma_U=\E \Sigmaeps \E^\top$. Since  $\Gamma(h)=\E^\top \Gamma(h)^\st \E$, an error bound for $\Gamma^\st(h)$ leads to   an error  bound for $\Gamma(h)$. 

\begin{thm} \label{thm.acf}
Let $\hat \A$ be some estimator of $\A$, $\Sigmah$ some estimator of $\Sigmaeps$, and $\hat \Gamma(h)^\st$  the analogue of \eqref{eq.Gamma.st} with $ \A$ and $ \Sigmaeps$ replaced by $\hat \A$ and $\Sigmah$. Furthermore, for any sub-multiplicative matrix norm $ \|\cdot \|$, let $\sum_{j=0}^\infty \|\A^j\|^2=:C_{\gamma,A}, \sum_{j=0}^\infty \|(\A^j)^\top\|^2=:C_{\gamma,A^\top}$, $ \sum_{j=0}^\infty \|\hat \A^j\|^2=:C_{\gamma,\hat A}$, $ \sum_{j=0}^\infty \|(\hat \A^j)^\top\|^2=:C_{\gamma,\hat A^\top}$ and $  \|\Sigmaeps\|=C_{\gamma,\Sigmaeps}$. Then, for $h\geq 0$
\begin{align*}
\|\hat\Gamma(h)-\Gamma(h)\| &\leq 
\ind_{(h\not=0)}\|\hat \A-\A\| (C_{\gamma,\hat A}+C_{\gamma, A}) \|\Gammas\|\\
&+\|\hat \A^h \|\Big(\|\hat \A -\A\| C_{\gamma,\Sigmaeps} (C_{\gamma,\hat A}+C_{\gamma, A})(C_{\gamma, A^\top}+C_{\gamma,A^\top})/4 +
\|\Sigmah -\Sigmaeps\| (C_{\gamma,\hat A}+C_{\gamma, A^\top})/2 \\
&+ \|\hat \A^\top -\A^\top\| (C_{\gamma,\Sigmaeps}+\|\Sigmah -\Sigmaeps\|) (C_{\gamma,\hat A}+C_{\gamma, A^\top})(C_{\gamma, \hat A}+C_{\gamma,\hat A^\top})/4\Big).
\end{align*}
\end{thm}

For  $\|\hat \A-\A\|$ small, we have $\sum_{l=0}^\infty \|\hat \A^l\| \leq \sum_{l=0}^\infty \|\A^l\|/(1-\|\hat \A-\A\| \sum_{s=0}^\infty \|\A^l\|)$. This  means that  $C_{\gamma,\hat A}$ and $C_{\gamma,\hat A^\top}$ can be bounded by $C_{\gamma,A}$ and $C_{\gamma,A^\top}$, respectively, and  Theorem~\ref{thm.acf} implies  that $\hat\Gamma(h)$ is a consistent estimator for $\Gamma(h)$ and that 
$$\|\hat \Gamma(h)-\Gamma(h)\|_\infty=O_P\Bigg((\sum_{j=0}^\infty \|\A^j\|_1+\sum_{j=0}^\infty \|\A^j\|_\infty)^2\|\Sigmaeps\|_1 \Big(\|\hat \A-\A\|_1+\|\hat \A-\A\|_\infty+\|\Sigmaeps-\Sigmah\|_1\Big)\Bigg).$$
Notice  that the term  $(\sum_{j=0}^\infty \|\A^j\|_1+\sum_{j=0}^\infty \|\A^j\|_\infty)^2\|\Sigmaeps\|_1$ depends on the VAR process and that this term  can be large. If $(A_1,\dots,A_p)\in \mathcal{M}(0,s,M,p)$ and $\Sigmaeps\in \mathcal{M}(0,s_\eps,M_\eps,1)$, this term is at least of the order 
$s^2 s_\eps$. Consequently, the sparsity setting  enabling a consistent autocovariance estimator with respect to the $\|\cdot\|_\infty$ norm, 
is more restrictive  than the sparsity setting 
enabling a  consistent parameter estimator with respect 
to the same norm.
In particular, if  we recall  the results of the Lasso estimator with Gaussian innovations and focus on sparsity and on the dimension of the system only, then we have
$\|\hat \Gamma(h)-\Gamma(h)\|_\infty=O_P(s^{3.5} s_\eps \sqrt{\log(dp)/(n-p)})$ in contrast to $\|\hat\A-\A\|_\infty=O_P(s\sqrt{\log(dp)/(n-p)})$.

We conclude this section with  a result related to the estimation of the  inverse of the spectral density matrix of the high dimensional VAR model considered.

\begin{thm} \label{thm.spec}
Let $\hat \A$ be some estimator of $\A$ and $\Sigmah^{-1}$ some estimator of $\Sigmaeps^{-1}$. If $(A_1,\dots,A_p)\in \mathcal{M}^{(2)}(q,s,M,p)$, $\Sigmaeps^{-1}\in \mathcal{M}(q_{\eps^{-1}},s_{\eps^{-1}},M_{\eps^{-1}},1)$ and for $l \in [1,\infty]$,  $\sum_{s=1}^p\|\hat A_s- A_s\|_l\leq t_{n,1}$ and  $\|\Sigmaeps^{-1}-\Sigmah^{-1}\|_l\leq t_{n,2}$, then
\begin{align*}
    \|f^{-1}(\omega)-\hat f^{-1}(\omega)\|_l\leq 2 M M_{\eps^{-1}} t_{n,1}+M^2 t_{n,2}+2M t_{n,1} t_{n,2}+t_{n,1}^2 M_{\eps{-1}}+ t_{n,1}^2 t_{n,2} ,
\end{align*}
where \begin{align}
    \hat f^{-1}(\omega)& =(I_d-\sum_{s=1}^p \hat A_s \exp(is \omega))^{\top} \Sigmah^{-1} (I_d-\sum_{s=1}^p \hat A_s \exp(-is \omega)\nonumber \\
    & =\mathcal{\hat A}(\exp(i\omega))^\top \Sigmah^{-1}  \mathcal{\hat A}(\exp(-i\omega)). \label{eq.spec.estimate}
\end{align}
\end{thm}

We stress here the fact that  for consistency of the inverse of the spectral density matrix, the more restrictive  sparsity setting \eqref{eq.sparse.pattern2} is used. If we consider again the results of the Lasso estimator with Gaussian innovations and focus on sparsity and on dimension only, we get the bound  $\|f^{-1}(\omega)-\hat f^{-1}(\omega)\|_\infty=O_P(s^{2.5} s_{\eps^{-1}} \sqrt{\log(dp)/(n-p)})$.

\subsection{The effects of the VAR process parameters on the error bounds}\label{sec.3.1}
In the following, we discuss the effects  of the  parameters  $\Sigmaeps$ and $\A$ on the estimation performance of the Lasso and of the Dantzig estimator. For the vectorized Lasso estimator \eqref{eq.mse.basu} with Gaussian innovations,  we refer to  Proposition 4.3 in \cite{basu2015}. We focus here only on terms which depend  on 
 $\Sigmaeps$ and $\A$ while the effect of all other terms is summarized  using the notation  $\gp$, where this term  may 
differ from equation to equation. The following error bounds are obtained. For the Lasso  we have
$$\|A-\hat A^\KC\|_{\max}\leq C_{DEP}^\KC \|\Gammas^{-1}\|_2^{(2-q)/2} \gps$$
and for the Dantzig estimator
$$
\|A-\hat A^\CL\|_{\max}\leq C_{DEP}^\CL \|\Gammas^{-1}\|_2 \gps.$$ 
Notice that the terms $C_{DEP}^\CL$ and $C_{DEP}^\KC$ appearing in these expressions  may differ since \cite{masini2019regularized} and \cite{wu2016performance} use different dependence conditions  to derive their concentration inequalities and the corresponding error bounds. If the error bounds for the Lasso are derived under the dependence conditions used in \cite{wu2016performance}, i.e., under physical dependence, then $C_{DEP}^\KC$ can be chosen such that it is identical up to constants to $C_{DEP}^\CL$. Here 
 we focus  on the bound derived under the physical dependence condition. In this case  we get that   $$C_{DEP}^\CL\leq C (\sum_{j=0}^\infty \|\A^j\|_2 \|)^2 \|\Sigmaeps\|_2 \max_{\|v\|_2=1} (E| v^\top \Sigmaeps^{-1/2} \eps_1\|_2^\tau)^{2/\tau}.$$ Furthermore, we have by Proposition 2.3 in \cite{basu2015}, that 
 \begin{align*}
    \|\Gammas^{-1}\|_2&\leq \sup_{\omega\in [-\pi.\pi]} \| f(\omega)^{-1}\|_2=\sup_{\omega\in [-\pi.\pi]} \| (I_d-\sum_{s=1}^p A_s \exp(i\omega s))^\top \Sigmaeps^{-1} (I_d-\sum_{s=1}^p A_s \exp(-i\omega s))\|_2 \\
    &\leq (1+\sum_{s=1}^p \|A_s\|_2)^2 \|\Sigmaeps^{-1}\|_2.
\end{align*}
Hence,
\begin{align*}
\|A-\hat A\|_{\max} & \leq C \max_{\|v\|_2=1} (E| v^\top \Sigmaeps^{-1/2} \eps_1\|^\tau)^{2/\tau} \gps\\
& \ \ \ \times
(\sum_{j=0}^\infty \|\A^j\|_2 \|)^2  (1+\sum_{s=1}^p \|A_s\|_2)^{2-\tilde q} \|\Sigmaeps\|_2  \|\Sigmaeps^{-1}\|_2^{(2-\tilde q)/2},
\end{align*}
where $\tilde q=q$ for the Lasso and $\tilde q=0$ for the Dantzig estimator. 

The term $\max_{\|v\|_2=1} (E| v^\top \Sigmaeps^{-1/2} \eps_1\|^\tau)^{2/\tau}$ depends on the distribution of the innovations. If $\eps$ is Gaussian, this quantity is not affected by $\Sigmaeps$. Furthermore, we have $\|\Sigmaeps\|_2 \| \Sigmaeps^{-1}\|_2 \geq \max_i e_i^\top \Sigmaeps e_i/\min_i e_i^\top \Sigmaeps e_i$. This  means that the dependence among the innovations as well as different variances between the  components of the innovations vector,  could have a negative effect on   the behavior of  the estimators.

Regarding  the influence of $\A$, recall that 
 the decay rate of $\A^j$ depends on the largest absolute eigenvalue of $\A$. Hence, if the VAR system is highly persistent, i.e., the largest absolute eigenvalue of $\A$ is close to one, then the constant  $\sum_{j=0}^\infty \|\A^j\|_2$ can be large.

\section{Numerical results}
In this section, we investigate by means of simulations, the finite sample performance of the  estimation procedures discussed, i.e., of the lasso estimator (\ref{eq.mse.basu}), of the row wise lasso (\ref{eq.row.LASSO}) and of the Dantzig  estimator (\ref{eq.VAR.CLIME.comp}). We denote the estimator \eqref{eq.mse.basu} by  {\ESTVECLASSO}, the estimator \eqref{eq.row.LASSO2} by {\ESTROWLASSO}, and the estimator \eqref{eq.VAR.CLIME.comp} by {\ESTROWDANTZIG}. All results presented are based on implementations in  \emph{R} \citep{R}. To compute the Lasso estimators \eqref{eq.mse.basu} and \eqref{eq.row.LASSO2}, we use the package \emph{glmnet} \citep{glmnet}. {\ESTVECLASSO} uses as  weighting matrix the inverse of the estimated innovation variance based on the estimated residuals of the  {\ESTROWLASSO} estimator.
For the Dantzig estimator \eqref{eq.VAR.CLIME.comp}, we use the package \emph{fastclime} \citep{fastclime}. It is worth mentioning here  that the estimation procedures using  the aforementioned  implementations highly differ with respect to  computing time. For instance, in order to estimate a VAR$(1)$ model  of dimension $d=100$ using  $n=100$ observations and without parallel computing on a personal computer, {\ESTROWLASSO} requires approximately 5 seconds, {\ESTVECLASSO} approximately  4.3 minutes, and {\ESTROWDANTZIG} approximately  16 minutes. More advanced techniques in linear programming may speed up the computation of the Dantzig estimator; see for instance \cite{mazumder2019computing}. 

The three estimators considered are used plain as well as with the following three modifications:

\begin{itemize}
    \item[S:] We standardize all  input time series, i.e., we insert a weighting matrix $W$ in \eqref{eq.VARregression} where  $W$ is a diagonal matrix with diagonal entries  given by the estimated standard deviations of each  time series.
    We then apply each  estimation procedure  on the transformed data $\mathcal{Y}W^{-1}=\mathcal{X}(I_p \otimes W)^{-1}(I_p \otimes W) B W^{-1} +\mathcal{E}W^{-1}=\mathcal{\tilde X} \tilde B+\mathcal{\tilde E}$ and transform the obtained estimates  $\widehat{\widetilde B}$ back using  $\hat B=(I_p \otimes W)^{-1}\widehat{\widetilde B} W$. 
    \item[A:] We apply a second adaptive step, i.e., we run the estimators twice and in the second run we insert penalty weights. For the coefficient $B_{i,j}$, we use in a second round the  penalty   $1/(|\hat B^{(1)}_{i,j}|+1/\sqrt{n})$, where $\hat B^{(1)}_{i,j}$ denotes the estimated coefficients obtained in the first round.
    \item[T:] We threshold the estimates, i.e., the final estimate is obtained by $\THRl(\hat B)$. Here, we use the adaptive thresholding, that is,  $\THRl(z)=z(1-|\lambdan/z|^\nu)_+$ for $ \nu>3$.
 \end{itemize}
A combination of the three aforementioned  modifications also can be used and the particular  modification  applied  is denoted  by capitalized letters. For instance,  the notation  {\ESTROWLASSO} SA means that {\ESTROWLASSO} is used with standardized time series and a second adaptive step. Notice that  {\ESTVECLASSO} possesses one tuning parameter, while  {\ESTROWLASSO} and {\ESTROWDANTZIG} possess $d$ tuning parameters, i.e., one for each row. These tuning parameters are selected using 
 the Bayesian Information Criterion (BIC). Additionally, the Extended Regularized Information Criterion (ERIC), see \cite{hui2015tuning}, is used as a second option to select the tuning parameters.

To evaluate the  performance of the different estimation procedures compared, we use the following quantities:
\begin{enumerate}[i)]
        \item $\|\A-\hat \A \|_\infty$, i.e., the estimation error for the parameter matrix $\A$ with respect to the  $\|\cdot \|_\infty$ matrix norm.\label{item.perf.A}
\item  $\|\Gammah-\Gammas\|_\infty/\|\Gammas\|_\infty$, i.e., the relative estimation error for the lag zero autocovariance with respect to the  $\|\cdot \|_\infty$ matrix norm.\label{item.perf.Gam}
\item $\int \|f(\omega)-\hat f(\omega)\|_\infty d\omega/\int \|f(\omega)\|_\infty d\omega$, i.e., the relative integrated estimation error for the spectral density matrix with respect to the  $\|\cdot \|_\infty$ matrix norm. In our calculations,   integrals are  approximated by   sums over the corresponding Fourier frequencies.\label{item.perf.f}
\item $1/d\sum_{j=1}^d MSE(\hat X_{n+h;j})/\sigma_j^2$, where  $\sigma_j=\sqrt{\var(\eps_{1;j})}$ and $\hat X_{n+h;j}$ denotes the forecast of the $j$th element of $X_{n+h}$ using  $\hat A$ and $X_1,\dots,X_n$. That is,    the averaged forecast error is computed  which is measured by  the  mean squared error for the forecasting horizon $h$. The mean squared error is estimated  using  $1000$ Monte Carlo runs. \label{item.perf.forecast}
\end{enumerate}
In order to estimate the second-order characteristics, i.e., $\Gamma(0)$ and $f$, we need to estimate the innovations variance $\Sigmaeps$. For this we use the estimator  \eqref{eq.eps.est.thres} and the implementation given in the package \emph{FinCovRegularization} \citep{FinCovRegularization}, which uses cross-validation to select the threshold parameter. 

Additionally to  the comparison of the different estimators,  we also investigate  the influence of the data generating process on the performance of the estimators. For this, we consider two groups of examples. In the first group, we vary the variance matrix of the innovations and keep everything else fixed. In the second group, we vary the dimension, the sparsity and the persistence of the processes but keep the variance matrix of the innovations fixed. 

The data generating processes in the first group of examples are different VAR$(4)$ processes. These processes are of dimension $d=14$ and  the same parameter matrix $A \in \mathcal{M}^{(2)}(0,5,17,4)$ is used,  with largest absolute eigenvalue equal to  $0.8$. The innovations are Gaussian and four different variance matrices are considered: a diagonal matrix with homogeneous variances among the components (denoted as DM), i.e., $\Sigmaeps=I_d$,  a diagonal matrix with heterogeneous variances among the components (denoted as DT), $\Sigmaeps=\diag(1.88\times 10^{-02}, 2.61\times 10^{-03}, 4.40\times 10^{-03}, 3.04\times 10^{-06}, 1.58\times 10^{-06}, 3.99\times 10^{-03}, 1.51\times 10^{-05}, 2.51\times 10^{-05}, 1.34\times 10^{-06}, 1.03\times 10^{-02}, 4.32\times 10^{-03}, 9.77\times 10^{-06}, 3.93\times 10^{-05}, 2.03\times 10^{-06})$, a non-sparse  variance matrix with homogeneous variances among the components and largest eigenvalue of 2.5 and smallest of 0.21 (denoted as FM), and the same non-sparse  variance matrix but now with heterogeneous variances among the components as given in the second matrix leading to a largest eigenvalue of $1.92\times 10^{-2}$ and a smallest eigenvalue of $4.45\times 10^{-7}$ (denoted as FT). 

The data generating processes in the second group of examples are VAR$(1)$ processes. The processes are of different dimensions $d=10, 25, 50$, and $100$, the innovations are Gaussian with $\Sigmaeps=I_d$ and the parameter matrix $A$ is generated randomly with row- and column-wise maximal sparsity given by $s=1, 3, 5, 10$, i.e., $A\in \mathcal{M}(0,s,\widetilde M,1)$, where $\widetilde M$ may differ from matrix to matrix, and the largest absolute eigenvalue takes the values  $\rho=0.6, 0.8, 0.9, 0.95$. The random generation of $A$ is done in four steps. First, a random matrix with largest absolute eigenvalue less than one is generated. Second, the $d-s$  smallest coefficients in absolute value  within each row are set equal to zero. Third, the $d-s$  smallest coefficients in absolute value within each column are set equal to zero. Finally, the obtained sparse matrix is rescaled so that its largest absolute eigenvalue equals  $\rho$. Note that most of the coefficients in the third step are already zero due to the second step. If in the fourth step no scaling is possible, i.e., the eigenvalue of the obtained sparse matrix is zero, then we set $e_1^\top A e_1=\rho$ and we rescale the matrix. 

A summary of the  results obtained are shown in Table 1 to Table 5. Table 1 presents  results  for Example 1 and for all four performance criteria used. Table 2 to Table 5 present the results for Example 2  and  for each one of the four different  performance criteria separately. 
For all three estimators considered, standardizing the input leads in most cases  to a better performance of the estimator. Furthermore, including a second adaptive step also improves the performance of the estimators in most of the cases considered. Additional thresholding has in most cases no  negative effect on  the performance of the estimators. For this reason we present  for Example~1 only the estimates obtained after applying  all modifications discussed while for Example~2 we focus on the estimates  with standardized inputs, a second adaptive step, and (an optional) thresholding. Furthermore, {\ESTVECLASSO} performs much better with ERIC than with BIC and for this reason  we focus on this selection rule only for this estimator applied in  Example~2. In the following we discuss the results obtained separately for Example 1 and for  Example 2. 

\subsection{Example 1}
As mentioned, the underlying processes of this example  are VAR$(4)$ processes with four different variance matrices for the Gaussian innovations. Two of them have very heterogeneous variances among the components. The parameter matrix $B=(A_1,A_2,A_3,A_4)\in \R^{14\times 56}$ has a row- and column-wise sparsity of $5$ and has in total $25$,  non-zero coefficients. A sample size   of $n=100$ observations  is used in the example.  Given  this  sample size,
 $25$ non-zero coefficients may be considered to be  too many for the {\ESTVECLASSO} estimator to handle. More specifically, if we plug this parameter design  into the corresponding error bounds, we get  for {\ESTVECLASSO} the bound $\|B-\hat B\|_\infty\leq\|\veco(B)-\veco(\hat B)\|_1\leq 25 \sqrt{\log(d^2p)/n}C\approx 6.5 \times C$,  compared to the bound $\|B-\hat B\|_\infty \leq 5 \sqrt{\log(dp)/n}C \approx 1 \times C$,  for {\ESTROWLASSO} and {\ESTROWDANTZIG}.
 Notice  that the error bounds for {\ESTVECLASSO} are derived, using 
 $\|\veco(B)-\veco(\hat B)\|_1$ and  that $\|B-\hat B\|_\infty$ could be substantially smaller. 

In Section~\ref{sec.3.1}, we mentioned that the estimation error of the parameter matrix $A$ can be bounded among others by $\|\Sigmaeps\|_2 \|\Sigmaeps^{-1}\|_2$. This differs highly between the heterogeneous cases, in which we have $\|\Sigmaeps\|_2 \|\Sigmaeps^{-1}\|_2>10^{4}$, and the  homogeneous cases, in which we have $\|\Sigmaeps\|_2 \|\Sigmaeps^{-1}\|_2<10^{2}$. This means  that  $\|\Sigmaeps\|_2 \|\Sigmaeps^{-1}\|_2$ is at least 100 times higher for the heterogeneous cases than for the homogeneous ones, and we see  this difference also in the performance of the estimators. For all estimation procedures considered, the estimation error of the parameter matrix $A$, i.e.,  criterion \ref{item.perf.A}), is considerably higher (up to factor 10) for the heterogeneous cases than for the homogeneous ones. Interestingly, we observe this only for the estimation error $ \|A-\widehat{A}\|_\ell$.
 For  the second-order properties as well as for  forecasting,  the corrsponding errors are  affected  much less by the heterogeneity of the   variance of the innovations.    

For {\ESTVECLASSO} we observe, that standardizing the input (S) greatly improves its  performance. Furthermore, a second adaptive step (A) is also very beneficial. However, additional thresholding (T) has almost no effect. Furthermore, {\ESTVECLASSO} performs better with ERIC than with BIC. For the heterogeneous cases, i.e., FT and DT, {\ESTVECLASSO} SA ERIC is  among the best ones with respect to all four evaluation criteria \ref{item.perf.A}) to \ref{item.perf.forecast}), previously discussed. For the homogeneous case, {\ESTVECLASSO} SA ERIC performs  good in forecasting, i.e., with respect to criterion \ref{item.perf.forecast}), but considerably worse in estimating the second-order characteristics of the VAR process, i.e., with respect to  criteria \ref{item.perf.Gam}) and \ref{item.perf.f}). 

For {\ESTROWLASSO}, a second adaptive step (A) improves the performance. Regarding the estimation of the second-order properties, criteria  \ref{item.perf.Gam}) and \ref{item.perf.f}),  standardizing the input (S) is beneficial for all cases. When it comes to forecasting, standardizing  is only beneficial for the heterogeneous cases. {\ESTROWLASSO} performs better with BIC than with ERIC. Additional thresholding (T) seems to affect the performance only slightly with no clear tendency. For the estimation of the second-order properties, {\ESTROWLASSO} SA BIC performs close to the best one in all cases. For forecasting, {\ESTROWLASSO} SA BIC is close to the best one in the heterogeneous cases and {\ESTROWLASSO} A BIC performs close to the best one in the homogeneous cases.

The combination of standardizing the input (S) and a second adaptive step (A) greatly improves the performance of {\ESTROWDANTZIG}. Again, BIC is here the better option than ERIC. Additional thresholding (T) has almost no effect but in some cases it brings some improvements. {\ESTROWDANTZIG} TSA BIC is not among the best estimates   of the parameter matrix $ A$ itself, i.e, for criterion  \ref{item.perf.A}), but it is  best or close to the best one in all cases for the estimation of the second-order properties as well as for forecasting, i.e., for criteria  \ref{item.perf.Gam}), \ref{item.perf.f}), and \ref{item.perf.forecast}). 

\subsection{Example 2}
We focus on the results of the thresholded estimators with standardized input, a second adaptive step, i.e., on estimators denoted by TSA. Notice that the results presented in Table 2 to Table 5  give the  relative performances of the different estimates. That is,  for each case considered, the results of each estimators are divided with those of the best performing estimator. Hence the closer to one is the corresponding entry in the tables, the closer to the best performing estimator is the particular estimator. Additionally, and in  order to also have the information related to the level of its performance, we denote for  the best performing estimator and instead of $1.00$, the absolute value of its score in brackets. 

For the estimation or the parameter matrix $A$ itself and over all considered settings, {\ESTVECLASSO} TSA ERIC is best with  {\ESTROWLASSO} BIC TSA performing very  close to the best. Overall the  performance of   {\ESTROWLASSO} BIC TSA is no more that of  $3\%$ worse than that of the best performing estimator. When the  dimension of the VAR model is small, additional thresholding could have a negative effect  on the performance whereas for large dimensions it is the other way around. Here  {\ESTROWDANTZIG} and {\ESTROWLASSO} perform much better with BIC than with ERIC. The estimation error   of the best performing procedures increases with dimension and  decreases with increasing persistence. A change in the sparsity levels seems to have  a minor affect on  performance. 

For the estimation of the second-order properties, i.e., for criteria  \ref{item.perf.Gam}) and \ref{item.perf.f}), {\ESTROWLASSO} TSA BIC performs very good for all persistence,  all sparsity levels and  for all  dimensions considered. For {\ESTROWLASSO} and {\ESTROWDANTZIG} using  ERIC seems to lead to  worse  results compared to those obtained using  BIC. The best {\ESTVECLASSO} and {\ESTROWDANTZIG} estimates perform usually more than $10\%$ to $20\%$ worse than the estimates  {\ESTROWLASSO} TSA BIC. In the case of  a strong persistence level ($\rho=0.95$), {\ESTROWDANTZIG} may lead to  unstable results, that is, the modulus of the largest absolute eigenvalue of the estimated parameter matrix of VAR model may be greater than one. No correction to stability is used here and therefore,  these results lead to  estimates of the second-order characteristics  which are not satisfactory. The performance of {\ESTROWDANTZIG} seems to get worse  with increasing dimension whereas no clear tendency can be observed with respect to  the sparsity and  to the different  persistence levels considered. 

Regarding forecasting, both selection options, i.e., BIC and ERIC, lead for {\ESTROWLASSO} and {\ESTROWDANTZIG} to good results. The performance of the three estimation methods considered differs only slightly. {\ESTROWLASSO} SA ERIC performs best with  {\ESTVECLASSO} SA ERIC and {\ESTROWDANTZIG} TSA ERIC being very  close to the best performance. The difference  between {\ESTVECLASSO} and {\ESTROWLASSO} increases with increasing dimension and persistence level,  whereas no clear tendency  can be identified for the differences   between {\ESTROWDANTZIG} and {\ESTROWLASSO}. Here, thresholding has a slight negative effect on  the performance. Note, however, that in contrast to estimating second-order characteristics,  thresholding is not necessary. Therefore, 
the available theoretical results for  the estimators considered,  justify their application to forecasting without the need for the use of an  additional thresholding step.
The forecast error of the best procedures increases slightly with dimension and with  persistence. A change in the sparsity level seems to have a rather  minor affect on the results obtained. 

\subsection{Conclusions}
If one is interested in estimating the  second-order characteristics of a VAR system, {\ESTROWDANTZIG} seems to be a good choice  for the first example, while  {\ESTROWLASSO} performs much better for the same estimation problem and for the second example considered. For this reason we  suggest to use {\ESTROWLASSO} for this objective. Furthermore, we suggest to use  {\ESTROWLASSO} with the modifications TSA, i.e., with  thresholding,   standardizing the time series  and  incorporating a second adaptive step. A second adaptive step  improves considerably  the performance of this estimator and standardizing the  time series  helps to develop some robustness in the sense that the performance of this estimator is not largely affected by the  particular second order characteristics of the underlying processes. As we have seen in Theorem~\ref{thm.thres}, thresholding gives the theoretical justification for using  {\ESTROWLASSO} in order  to consistently estimate the second-order characteristics of the underlying VAR process. In the examples considered, thresholding  does not necessarily  improve the performance of the estimator but it enlarges the range of its applicability by gaining  consistency with respect to a  much  larger set  of matrix norms.  To select the tuning parameter for estimating second-order characteristics, our simulation study suggests  that BIC is the better option for {\ESTROWLASSO}, that is  {\ESTROWLASSO} TSA BIC is the suggested estimation method to use for estimating   second-order properties.

If the main interest is forecasting, all three estimators perform well and there is no one which clearly outperforms the others. Since valid forecasts can be obtained with less consistency requirements on the estimators compare to those needed for consistent estimation  of second order characteristics,  an additional thresholding may be omitted in this case. Based on the first example, {\ESTVECLASSO} SA ERIC and {\ESTROWDANTZIG} TSA BIC seems to be slightly more robust than {\ESTROWLASSO} TSA BIC. Our findings also suggest that if  {\ESTVECLASSO} is used, then  ERIC should be preferred to  BIC for selecting the tuning parameter. Notice, however, that  the existing  theory for  {\ESTVECLASSO} does not cover all sparsity settings considered in our simulation study. 

\begin{sidewaystable}[ht]
\begin{tabular}{|l|rrrr|rrrr|rrrr|rrrr|}
  \hline   & \multicolumn{4}{c|}{$\|\hat A-A\|_\infty$} &   \multicolumn{4}{c|}{$\|\Gammah-\Gammas\|_1/\|\Gammas\|_1$} &  \multicolumn{4}{c|}{$\int \|f(\omega)-\hat f(\omega)\|_1d\omega/\int \|f(\omega)\|_1d\omega$} & \multicolumn{4}{|c|}{$1/d\sum_{j=1}^d MSE(\hat X_{n+1;j})/\sigma_j^2$} \\ & FM & FT & DM & DT &  FM & FT & DM & DT &  FM & FT & DM & DT & FM & FT & DM & DT \\   \hline
 {\ESTVECLASSO} & 14.47 & 14.50 & 14.40 & 14.52 & 1.07 & 0.67 & 3.91 & 0.65 & 0.91 & 0.70 & 0.91 & 0.69 & 70.96 & 2.35 & 63.31 & 2.47 \\ 
  {\ESTVECLASSO} S & 2.16 & 9.99 & 2.05 & 10.10 & 0.77 & 0.65 & 0.72 & 0.61 & 0.76 & 0.67 & 0.71 & 0.60 & 1.42 & 1.21 & 1.42 & 1.21 \\ 
  {\ESTVECLASSO} A & 14.49 & 14.47 & 14.40 & 14.47 & 0.80 & 0.66 & 0.80 & 0.60 & 0.82 & 0.69 & 0.82 & 0.64 & 38.45 & 2.02 & 35.12 & 2.11 \\ 
  {\ESTVECLASSO} SA & 1.90 & 9.09 & 1.86 & 9.59 & 0.74 & 0.56 & 0.67 & 0.50 & 0.73 & 0.57 & 0.66 & 0.48 & 1.27 & 1.11 & 1.25 & 1.12 \\ 
  {\ESTVECLASSO} TA & 14.49 & 14.47 & 14.40 & 14.47 & 0.81 & 1.07 & 0.80 & 0.66 & 0.83 & 0.68 & 0.82 & 0.63 & 39.03 & 1.89 & 35.67 & 1.93 \\ 
  {\ESTVECLASSO} TSA & 1.84 & 10.58 & 1.84 & 12.12 & 0.74 & 0.56 & 0.67 & 0.50 & 0.74 & 0.57 & 0.66 & 0.48 & 1.28 & 1.13 & 1.26 & 1.14 \\ 
  {\ESTVECLASSO} S ERIC & 1.98 & 11.10 & 1.90 & 11.12 & 0.72 & 0.63 & 0.67 & 0.60 & 0.71 & 0.64 & 0.67 & 0.60 & 1.30 & 1.17 & 1.31 & 1.17 \\ 
  {\ESTVECLASSO} SA ERIC & 1.62 & 9.29 & 1.59 & \textbf{9.54} & 0.69 & 0.55 & 0.63 & 0.51 & 0.68 & \textbf{0.56} & 0.63 & 0.50 & 1.20 & \textbf{1.10} & 1.20 & \textbf{1.10} \\ 
  {\ESTVECLASSO} TSA ERIC & 1.53 & \textbf{8.86} & 1.53 & 9.56 & 0.70 & 0.55 & 0.64 & 0.48 & 0.69 & \textbf{0.56} & 0.63 & \textbf{0.46} & 1.20 & 1.11 & 1.20 & 1.11 \\ 
   \hline
{\ESTROWLASSO} & 1.94 & 14.53 & 1.74 & 14.54 & 0.66 & 0.62 & 0.69 & 0.63 & 0.69 & 0.67 & 0.71 & 0.68 & 1.42 & 2.11 & 1.41 & 2.15 \\ 
  {\ESTROWLASSO} S & 2.12 & 12.33 & 2.05 & 12.15 & 0.71 & 0.66 & 0.72 & 0.63 & 0.71 & 0.68 & 0.72 & 0.63 & 1.54 & 1.23 & 1.48 & 1.23 \\ 
  {\ESTROWLASSO} A & \textbf{1.24} & 14.50 & \textbf{1.22} & 14.50 & 0.51 & 0.57 & 0.51 & 0.56 & 0.59 & 0.62 & 0.59 & 0.61 & 1.18 & 1.86 & 1.18 & 1.89 \\ 
  {\ESTROWLASSO} SA & 1.49 & 11.46 & 1.39 & 11.65 & 0.47 & 0.55 & 0.46 & 0.52 & 0.48 & 0.58 & 0.47 & 0.53 & 1.35 & 1.11 & 1.30 & 1.11 \\ 
  {\ESTROWLASSO} TA & 1.52 & 14.50 & 1.57 & 14.50 & 0.87 & 0.57 & 0.87 & 0.56 & 1.25 & 0.62 & 1.24 & 0.61 & 578.12 & 1.86 & 458.47 & 1.89 \\ 
  {\ESTROWLASSO} TSA & 1.42 & 11.07 & 1.33 & 11.05 & 0.48 & 0.55 & 0.47 & 0.50 & 0.49 & 0.57 & 0.48 & 0.50 & 1.35 & 1.11 & 1.30 & 1.11 \\ 
  {\ESTROWLASSO} S ERIC & 2.22 & 56.90 & 2.20 & 60.73 & 0.68 & 0.76 & 0.69 & 0.82 & 0.68 & 0.95 & 0.69 & 1.07 & 1.57 & 1.97 & 1.51 & 2.03 \\ 
  {\ESTROWLASSO} SA ERIC & 1.77 & 60.39 & 1.76 & 62.45 & 0.49 & 0.95 & 0.48 & 0.94 & 0.52 & 1.30 & 0.51 & 1.38 & 1.43 & 2.00 & 1.39 & 2.06 \\ 
  {\ESTROWLASSO} TSA ERIC & 1.72 & 59.88 & 1.71 & 61.62 & 0.52 & 0.97 & 0.50 & 0.93 & 0.54 & 1.31 & 0.52 & 1.36 & 1.42 & 1.99 & 1.38 & 2.06 \\ 
   \hline
{\ESTROWDANTZIG} & 2.53 & 14.49 & 2.20 & 14.53 & 0.60 & 0.53 & 0.63 & 0.47 & 0.66 & 0.62 & 0.67 & 0.57 & 1.49 & 6.08 & 1.45 & 5.47 \\ 
  {\ESTROWDANTZIG} S & 2.59 & 13.74 & 2.23 & 13.49 & 0.61 & 0.64 & 0.62 & 0.62 & 0.61 & 0.65 & 0.62 & 0.63 & 1.27 & 1.20 & 1.24 & 1.20 \\ 
  {\ESTROWDANTZIG} A & 1.69 & 14.49 & 1.56 & 14.52 & 0.48 & \textbf{0.52} & 0.48 & \textbf{0.45} & 0.60 & 0.61 & 0.61 & 0.54 & 1.22 & 6.08 & 1.19 & 5.47 \\ 
  {\ESTROWDANTZIG} SA & 2.01 & 12.00 & 1.77 & 11.88 & \textbf{0.42} & 0.54 & \textbf{0.42} & 0.50 & \textbf{0.44} & 0.57 & 0.43 & 0.52 & \textbf{1.14} & 1.11 & \textbf{1.12} & \textbf{1.10} \\ 
  {\ESTROWDANTZIG} TA & 2.37 & 14.49 & 2.10 & 14.52 & 0.81 & \textbf{0.52} & 0.80 & \textbf{0.45} & 1.15 & 0.61 & 1.11 & 0.54 & 500.57 & 6.08 & 394.63 & 5.47 \\ 
  {\ESTROWDANTZIG} TSA & 1.99 & 11.65 & 1.76 & 11.57 & \textbf{0.42} & 0.54 & \textbf{0.42} & 0.48 & \textbf{0.44} & \textbf{0.56} & \textbf{0.42} & 0.49 & \textbf{1.14} & \textbf{1.10} & \textbf{1.12} & \textbf{1.10} \\ 
  {\ESTROWDANTZIG} S ERIC & 2.67 & 50.81 & 2.33 & 57.45 & 0.56 & 0.91 & 0.57 & 0.92 & 0.57 & 1.20 & 0.58 & 1.32 & 1.30 & 1.71 & 1.27 & 1.79 \\ 
  {\ESTROWDANTZIG} SA ERIC & 2.17 & 50.37 & 1.96 & 55.83 & 0.43 & 1.11 & 0.43 & 1.07 & 0.48 & 1.50 & 0.48 & 1.62 & 1.20 & 1.68 & 1.19 & 1.78 \\ 
  {\ESTROWDANTZIG} TSA ERIC & 2.16 & 50.02 & 1.95 & 55.42 & 0.43 & 1.13 & 0.43 & 1.07 & 0.48 & 1.52 & 0.47 & 1.62 & 1.20 & 1.68 & 1.19 & 1.77 \\ 
  \hline
  \end{tabular}
\caption{Example 1 -- VAR$(4)$,  $d=14, \rho=0.8, s=5, n=100$} 
\end{sidewaystable}

\begin{sidewaystable}[ht]
\centering
\begin{tabular}{|l|l|rrrr|rrrr|rrrr|rrrr|}
  \hline & $s$   & \multicolumn{4}{c|}{$1$}&\multicolumn{4}{c|}{$3$} &\multicolumn{4}{c|}{$5$} &\multicolumn{4}{c|}{$10$}    \\ \hline & \diagbox{$d$}{$\rho$}  & 0.6 & 0.8 & 0.9 & 0.95 & 0.6 & 0.8 & 0.9 & 0.95 & 0.6 & 0.8 & 0.9 & 0.95 & 0.6 & 0.8 & 0.9 & 0.95 \\ \hline 
 {\ESTVECLASSO} SA ERIC & \multirow{7}{0.5cm}{10} & \textminbf {0.52} & 1.02 & 1.04 & 1.04 & \textminbf {0.82} & \textminbf {0.71} & 1.03 & \textminbf {0.75} & \textminbf {0.58} & 1.02 & 1.04 & 1.04 & \textminbf {0.87} & \textminbf {0.86} & \textminbf {0.75} & 1.01 \\ 
  {\ESTVECLASSO} TSA ERIC &  & \textminbf {0.52} & \textminbf {0.43} & \textminbf {0.46} & \textminbf {0.46} & 1.06 & \textminbf {0.71} & 1.03 & \textminbf {0.75} & 1.02 & \textminbf {0.55} & \textminbf {0.47} & \textminbf {0.45} & 1.07 & 1.02 & \textminbf {0.75} & \textminbf {0.72} \\ 
  {\ESTROWLASSO} SA BIC &  & 1.15 & 1.09 & 1.02 & 1.02 & 1.13 & 1.03 & 1.01 & \textminbf {0.75} & 1.16 & 1.02 & 1.02 & \textminbf {0.45} & 1.14 & 1.03 & 1.01 & \textminbf {0.72} \\ 
  {\ESTROWLASSO} TSA ERIC &  & 1.23 & 1.28 & 1.15 & 1.15 & 1.02 & \textminbf {0.71} & \textminbf {0.73} & \textminbf {0.75} & 1.17 & 1.09 & 1.15 & 1.16 & 1.02 & \textminbf {0.86} & \textminbf {0.75} & \textminbf {0.72} \\ 
  {\ESTROWLASSO} TSA BIC &  & 1.17 & 1.12 & 1.02 & \textminbf {0.46} & 1.17 & 1.06 & 1.01 & \textminbf {0.75} & 1.19 & 1.04 & 1.02 & \textminbf {0.45} & 1.18 & 1.07 & 1.01 & \textminbf {0.72} \\ 
  {\ESTROWDANTZIG} TSA ERIC &  & 1.21 & 1.30 & 1.22 & 1.22 & 1.05 & 1.07 & 1.07 & 1.04 & 1.16 & 1.15 & 1.21 & 1.22 & 1.05 & 1.05 & 1.05 & 1.06 \\ 
  {\ESTROWDANTZIG} TSA BIC &  & 1.17 & 1.14 & 1.11 & 1.09 & 1.21 & 1.14 & 1.10 & 1.04 & 1.19 & 1.09 & 1.09 & 1.09 & 1.22 & 1.12 & 1.08 & 1.06 \\ 
   \hline
{\ESTVECLASSO} SA ERIC & \multirow{7}{0.5cm}{25} & \textminbf {1.03} & 1.01 & 1.08 & 1.16 & \textminbf {1.03} & 1.01 & 1.08 & 1.14 & \textminbf {1.04} & 1.01 & 1.10 & 1.16 & \textminbf {1.09} & 1.01 & 1.07 & 1.16 \\ 
  {\ESTVECLASSO} TSA ERIC &  & 1.09 & \textminbf {0.91} & 1.05 & 1.12 & 1.08 & \textminbf {0.90} & 1.04 & 1.09 & 1.08 & \textminbf {1.03} & 1.06 & 1.12 & 1.08 & \textminbf {0.90} & 1.05 & 1.12 \\ 
  {\ESTROWLASSO} SA BIC &  & 1.16 & 1.04 & 1.02 & 1.03 & 1.15 & 1.04 & 1.03 & 1.03 & 1.14 & 1.05 & 1.02 & 1.03 & 1.16 & 1.07 & 1.01 & 1.03 \\ 
  {\ESTROWLASSO} TSA ERIC &  & 1.19 & 1.11 & 1.12 & 1.14 & 1.20 & 1.11 & 1.13 & 1.14 & 1.17 & 1.07 & 1.14 & 1.15 & 1.17 & 1.12 & 1.11 & 1.13 \\ 
  {\ESTROWLASSO} TSA BIC &  & 1.17 & 1.04 & \textminbf {0.83} & \textminbf {0.76} & 1.16 & 1.04 & \textminbf {0.79} & \textminbf {0.74} & 1.16 & 1.06 & \textminbf {0.81} & \textminbf {0.73} & 1.17 & 1.06 & \textminbf {0.83} & \textminbf {0.75} \\ 
  {\ESTROWDANTZIG} TSA ERIC &  & 1.24 & 1.20 & 1.22 & 1.22 & 1.24 & 1.19 & 1.22 & 1.20 & 1.22 & 1.15 & 1.22 & 1.23 & 1.22 & 1.20 & 1.19 & 1.25 \\ 
  {\ESTROWDANTZIG} TSA BIC &  & 1.20 & 1.16 & 1.19 & 1.22 & 1.20 & 1.17 & 1.19 & 1.22 & 1.20 & 1.16 & 1.21 & 1.23 & 1.20 & 1.17 & 1.18 & 1.24 \\ 
   \hline
{\ESTVECLASSO} SA ERIC & \multirow{7}{0.5cm}{50} & \textminbf {1.20} & \textminbf {1.07} & 1.03 & 1.05 & \textminbf {1.13} & \textminbf {1.00} & 1.04 & 1.05 & \textminbf {1.14} & \textminbf {1.06} & 1.03 & 1.07 & \textminbf {1.11} & \textminbf {1.03} & 1.03 & 1.07 \\ 
  {\ESTVECLASSO} TSA ERIC &  & 1.08 & 1.02 & \textminbf {1.07} & \textminbf {0.93} & 1.09 & 1.03 & 1.01 & \textminbf {1.00} & 1.08 & 1.02 & \textminbf {1.06} & 1.02 & 1.08 & 1.02 & \textminbf {1.06} & 1.02 \\ 
  {\ESTROWLASSO} SA BIC &  & 1.11 & 1.09 & 1.03 & 1.08 & 1.13 & 1.07 & 1.03 & 1.05 & 1.13 & 1.08 & 1.06 & 1.05 & 1.12 & 1.09 & 1.03 & 1.04 \\ 
  {\ESTROWLASSO} TSA ERIC &  & 1.74 & 1.58 & 1.34 & 1.42 & 1.81 & 1.67 & 1.31 & 1.43 & 1.82 & 1.54 & 1.37 & 1.39 & 1.80 & 1.63 & 1.31 & 1.33 \\ 
  {\ESTROWLASSO} TSA BIC &  & 1.09 & 1.06 & 1.01 & 1.02 & 1.12 & 1.01 & \textminbf {1.04} & 1.01 & 1.10 & 1.05 & 1.03 & \textminbf {1.01} & 1.09 & 1.04 & 1.01 & \textminbf {1.03} \\ 
  {\ESTROWDANTZIG} TSA ERIC &  & 1.76 & 1.70 & 1.50 & 1.62 & 1.84 & 1.77 & 1.51 & 1.61 & 1.89 & 1.68 & 1.49 & 1.55 & 1.91 & 1.76 & 1.47 & 1.50 \\ 
  {\ESTROWDANTZIG} TSA BIC &  & 1.13 & 1.17 & 1.17 & 1.25 & 1.15 & 1.16 & 1.19 & 1.27 & 1.17 & 1.18 & 1.20 & 1.26 & 1.14 & 1.13 & 1.20 & 1.26 \\ 
   \hline
{\ESTVECLASSO} SA ERIC & \multirow{7}{0.5cm}{100} & 1.05 & 1.08 & 1.10 & 1.14 & 1.04 & 1.06 & 1.09 & 1.14 & 1.04 & 1.04 & 1.07 & 1.15 & 1.04 & 1.06 & 1.10 & 1.16 \\ 
  {\ESTVECLASSO} TSA ERIC &  & 1.01 & 1.03 & \textminbf {1.16} & \textminbf {1.05} & 1.01 & 1.01 & \textminbf {1.15} & \textminbf {1.10} & 1.01 & \textminbf {1.39} & \textminbf {1.24} & \textminbf {1.08} & \textminbf {1.58} & \textminbf {1.34} & \textminbf {1.15} & \textminbf {1.05} \\ 
  {\ESTROWLASSO} SA BIC &  & 1.03 & 1.05 & 1.09 & 1.14 & 1.04 & 1.04 & 1.13 & 1.11 & 1.04 & 1.06 & 1.07 & 1.12 & 1.04 & 1.08 & 1.09 & 1.14 \\ 
  {\ESTROWLASSO} TSA ERIC &  & 3.17 & 3.23 & 3.13 & 3.29 & 3.16 & 3.18 & 3.13 & 3.07 & 3.13 & 3.00 & 2.96 & 3.30 & 3.03 & 3.14 & 3.07 & 3.46 \\ 
  {\ESTROWLASSO} TSA BIC &  & \textminbf {1.49} & \textminbf {1.32} & 1.02 & 1.08 & \textminbf {1.51} & \textminbf {1.34} & 1.07 & 1.05 & \textminbf {1.50} & 1.01 & 1.01 & 1.06 & 1.01 & 1.02 & 1.01 & 1.08 \\ 
  {\ESTROWDANTZIG} TSA ERIC &  & 2.91 & 2.81 & 2.67 & 2.81 & 2.89 & 2.78 & 2.70 & 2.63 & 2.88 & 2.65 & 2.50 & 2.87 & 2.78 & 2.76 & 2.57 & 2.99 \\ 
  {\ESTROWDANTZIG} TSA BIC &  & 1.05 & 1.06 & 1.12 & 1.19 & 1.04 & 1.06 & 1.16 & 1.17 & 1.05 & 1.09 & 1.15 & 1.15 & 1.03 & 1.06 & 1.12 & 1.18 \\ 
   \hline
\end{tabular}
\caption{Example 2 -- VAR(1), $\|\hat A-A\|_\infty$, n=100} 
\end{sidewaystable}
\begin{sidewaystable}[ht]
\centering
\begin{tabular}{|l|l|rrrr|rrrr|rrrr|rrrr|}
  \hline & $s$   & \multicolumn{4}{c|}{$1$}&\multicolumn{4}{c|}{$3$} &\multicolumn{4}{c|}{$5$} &\multicolumn{4}{c|}{$10$}    \\ \hline & \diagbox{$d$}{$\rho$}  & 0.6 & 0.8 & 0.9 & 0.95 & 0.6 & 0.8 & 0.9 & 0.95 & 0.6 & 0.8 & 0.9 & 0.95 & 0.6 & 0.8 & 0.9 & 0.95 \\ \hline 
 {\ESTVECLASSO} SA ERIC & \multirow{7}{0.5cm}{10} & 1.02 & 1.05 & 1.10 & 1.07 & 1.08 & 1.16 & 1.20 & 1.14 & 1.02 & 1.11 & 1.08 & 1.04 & 1.08 & 1.15 & 1.19 & 1.17 \\ 
  {\ESTVECLASSO} TSA ERIC &  & 1.02 & 1.05 & 1.10 & 1.07 & 1.08 & 1.16 & 1.20 & 1.14 & 1.02 & 1.11 & 1.11 & 1.07 & 1.10 & 1.15 & 1.19 & 1.20 \\ 
  {\ESTROWLASSO} SA BIC &  & 1.02 & \textminbf {0.37} & \textminbf {0.39} & \textminbf {0.46} & \textminbf {0.39} & \textminbf {0.32} & \textminbf {0.35} & \textminbf {0.43} & \textminbf {0.42} & \textminbf {0.38} & \textminbf {0.38} & \textminbf {0.46} & 1.03 & \textminbf {0.34} & \textminbf {0.36} & \textminbf {0.40} \\ 
  {\ESTROWLASSO} TSA ERIC &  & \textminbf {0.41} & \textminbf {0.37} & 1.05 & 1.11 & \textminbf {0.39} & 1.03 & 1.06 & \textminbf {0.43} & \textminbf {0.42} & 1.03 & 1.05 & 1.04 & \textminbf {0.39} & 1.03 & 1.06 & 1.02 \\ 
  {\ESTROWLASSO} TSA BIC &  & 1.02 & \textminbf {0.37} & 1.03 & \textminbf {0.46} & 1.03 & \textminbf {0.32} & \textminbf {0.35} & 1.02 & 1.02 & 1.03 & 1.03 & \textminbf {0.46} & 1.03 & \textminbf {0.34} & \textminbf {0.36} & \textminbf {0.40} \\ 
  {\ESTROWDANTZIG} TSA ERIC &  & 1.02 & 1.11 & 1.13 & 1.22 & 1.05 & 1.09 & 1.11 & 1.09 & 1.02 & 1.11 & 1.16 & 1.17 & 1.05 & 1.06 & 1.11 & 1.10 \\ 
  {\ESTROWDANTZIG} TSA BIC &  & 1.05 & 1.11 & 1.15 & 1.91 & 1.08 & 1.09 & 1.11 & 1.12 & 1.05 & 1.11 & 1.16 & 1.30 & 1.10 & 1.06 & 1.08 & 1.12 \\ 
   \hline
{\ESTVECLASSO} SA ERIC & \multirow{7}{0.5cm}{25} & 1.02 & 1.20 & 1.26 & 1.23 & 1.04 & 1.17 & 1.25 & 1.26 & 1.02 & 1.17 & 1.31 & 1.26 & 1.04 & 1.15 & 1.29 & 1.26 \\ 
  {\ESTVECLASSO} TSA ERIC &  & 1.04 & 1.20 & 1.28 & 1.28 & 1.02 & 1.17 & 1.25 & 1.28 & 1.04 & 1.17 & 1.31 & 1.28 & 1.04 & 1.15 & 1.31 & 1.30 \\ 
  {\ESTROWLASSO} SA BIC &  & \textminbf {0.56} & \textminbf {0.45} & \textminbf {0.43} & \textminbf {0.47} & 1.02 & \textminbf {0.46} & \textminbf {0.44} & \textminbf {0.46} & \textminbf {0.56} & \textminbf {0.47} & \textminbf {0.42} & \textminbf {0.47} & \textminbf {0.57} & \textminbf {0.46} & \textminbf {0.42} & \textminbf {0.47} \\ 
  {\ESTROWLASSO} TSA ERIC &  & 1.04 & 1.04 & 1.02 & 1.04 & 1.04 & 1.02 & 1.02 & 1.02 & 1.02 & 1.02 & 1.05 & 1.02 & 1.02 & 1.04 & 1.05 & 1.04 \\ 
  {\ESTROWLASSO} TSA BIC &  & \textminbf {0.56} & 1.02 & \textminbf {0.43} & 1.02 & \textminbf {0.56} & \textminbf {0.46} & \textminbf {0.44} & \textminbf {0.46} & \textminbf {0.56} & \textminbf {0.47} & \textminbf {0.42} & \textminbf {0.47} & \textminbf {0.57} & \textminbf {0.46} & 1.02 & 1.02 \\ 
  {\ESTROWDANTZIG} TSA ERIC &  & 1.04 & 1.09 & 1.14 & 4.91 & 1.04 & 1.09 & 1.09 & 1.35 & 1.04 & 1.06 & 1.14 & 5.89 & 1.02 & 1.09 & 1.14 & 1.23 \\ 
  {\ESTROWDANTZIG} TSA BIC &  & 1.04 & 1.11 & 1.16 & 1.47 & 1.05 & 1.11 & 1.18 & 1.24 & 1.04 & 1.09 & 1.14 & 1.19 & 1.05 & 1.11 & 1.12 & 1.34 \\ 
   \hline
{\ESTVECLASSO} SA ERIC & \multirow{7}{0.5cm}{50} & 1.03 & 1.19 & 1.32 & 1.27 & 1.05 & 1.20 & 1.32 & 1.25 & 1.07 & 1.20 & 1.30 & 1.20 & 1.08 & 1.20 & 1.33 & 1.28 \\ 
  {\ESTVECLASSO} TSA ERIC &  & 1.11 & 1.23 & 1.32 & 1.29 & 1.17 & 1.22 & 1.32 & 1.27 & 1.12 & 1.25 & 1.32 & 1.22 & 1.12 & 1.22 & 1.35 & 1.30 \\ 
  {\ESTROWLASSO} SA BIC &  & \textminbf {0.64} & \textminbf {0.52} & \textminbf {0.50} & \textminbf {0.52} & 1.02 & \textminbf {0.49} & \textminbf {0.50} & \textminbf {0.52} & \textminbf {0.59} & \textminbf {0.51} & \textminbf {0.50} & \textminbf {0.55} & 1.02 & \textminbf {0.50} & \textminbf {0.49} & \textminbf {0.53} \\ 
  {\ESTROWLASSO} TSA ERIC &  & 1.11 & 1.04 & 1.02 & 1.06 & 1.17 & 1.08 & \textminbf {0.50} & 1.04 & 1.20 & 1.08 & \textminbf {0.50} & 1.04 & 1.19 & 1.06 & \textminbf {0.49} & 1.04 \\ 
  {\ESTROWLASSO} TSA BIC &  & 1.02 & \textminbf {0.52} & 1.02 & 1.02 & \textminbf {0.60} & \textminbf {0.49} & 1.02 & 1.02 & \textminbf {0.59} & \textminbf {0.51} & \textminbf {0.50} & \textminbf {0.55} & \textminbf {0.59} & \textminbf {0.50} & 1.02 & 1.02 \\ 
  {\ESTROWDANTZIG} TSA ERIC &  & 1.09 & 1.10 & 1.06 & 1.25 & 1.15 & 1.14 & 1.04 & 1.15 & 1.17 & 1.16 & 1.06 & 1.33 & 1.15 & 1.14 & 1.08 & 1.06 \\ 
  {\ESTROWDANTZIG} TSA BIC &  & 1.05 & 1.13 & 1.10 & 1.65 & 1.08 & 1.16 & 1.14 & 1.21 & 1.05 & 1.16 & 1.20 & 1.18 & 1.10 & 1.14 & 1.12 & 1.17 \\ 
   \hline
{\ESTVECLASSO} SA ERIC & \multirow{7}{0.5cm}{100} & 1.09 & 1.21 & 1.30 & 1.26 & 1.10 & 1.21 & 1.27 & 1.29 & 1.06 & 1.21 & 1.27 & 1.22 & 1.09 & 1.19 & 1.27 & 1.23 \\ 
  {\ESTVECLASSO} TSA ERIC &  & 1.13 & 1.21 & 1.30 & 1.26 & 1.12 & 1.21 & 1.27 & 1.31 & 1.10 & 1.22 & 1.27 & 1.22 & 1.13 & 1.21 & 1.27 & 1.25 \\ 
  {\ESTROWLASSO} SA BIC &  & 1.01 & \textminbf {0.58} & \textminbf {0.54} & \textminbf {0.58} & \textminbf {0.69} & 1.02 & \textminbf {0.55} & \textminbf {0.58} & \textminbf {0.69} & \textminbf {0.58} & \textminbf {0.56} & \textminbf {0.59} & \textminbf {0.70} & \textminbf {0.58} & \textminbf {0.55} & \textminbf {0.57} \\ 
  {\ESTROWLASSO} TSA ERIC &  & 1.72 & 1.29 & 1.15 & 1.09 & 1.72 & 1.35 & 1.15 & 1.09 & 1.68 & 1.28 & 1.12 & 1.10 & 1.64 & 1.33 & 1.15 & 1.16 \\ 
  {\ESTROWLASSO} TSA BIC &  & \textminbf {0.68} & \textminbf {0.58} & \textminbf {0.54} & 1.02 & 1.01 & \textminbf {0.57} & 1.02 & 1.02 & \textminbf {0.69} & 1.02 & 1.02 & 1.02 & 1.01 & \textminbf {0.58} & \textminbf {0.55} & 1.02 \\ 
  {\ESTROWDANTZIG} TSA ERIC &  & 1.49 & 1.24 & 1.22 & 1.28 & 1.49 & 1.30 & 1.20 & 1.17 & 1.43 & 1.26 & 1.20 & 1.36 & 1.43 & 1.33 & 1.27 & 1.23 \\ 
  {\ESTROWDANTZIG} TSA BIC &  & 1.04 & 1.12 & 1.15 & 1.47 & 1.04 & 1.16 & 1.16 & 1.14 & 1.06 & 1.16 & 1.12 & 1.42 & 1.09 & 1.10 & 1.15 & 1.23 \\ 
   \hline
\end{tabular}
\caption{Example 2 -- VAR(1), $\|\Gammah-\Gammas\|_\infty/\|\Gammas\|_\infty$, n=100} 
\end{sidewaystable}
\begin{sidewaystable}[ht]
\centering
\begin{tabular}{|l|l|rrrr|rrrr|rrrr|rrrr|}
  \hline & $s$   & \multicolumn{4}{c|}{$1$}&\multicolumn{4}{c|}{$3$} &\multicolumn{4}{c|}{$5$} &\multicolumn{4}{c|}{$10$}    \\ \hline & \diagbox{$d$}{$\rho$}  & 0.6 & 0.8 & 0.9 & 0.95 & 0.6 & 0.8 & 0.9 & 0.95 & 0.6 & 0.8 & 0.9 & 0.95 & 0.6 & 0.8 & 0.9 & 0.95 \\ \hline 
 {\ESTVECLASSO} SA ERIC & \multirow{7}{0.5cm}{10} & \textminbf {0.49} & 1.05 & 1.07 & 3.10 & 1.07 & 1.11 & 1.13 & 1.07 & \textminbf {0.49} & 1.07 & 1.05 & 1.73 & 1.07 & 1.10 & 1.13 & 1.12 \\ 
  {\ESTVECLASSO} TSA ERIC &  & 1.02 & 1.05 & 1.10 & 2.25 & 1.11 & 1.11 & 1.16 & 1.09 & 1.02 & 1.10 & 1.07 & 1.73 & 1.11 & 1.12 & 1.13 & 1.14 \\ 
  {\ESTROWLASSO} SA BIC &  & \textminbf {0.49} & \textminbf {0.42} & \textminbf {0.42} & 1.04 & 1.04 & \textminbf {0.38} & \textminbf {0.38} & 1.02 & \textminbf {0.49} & \textminbf {0.41} & \textminbf {0.42} & \textminbf {0.51} & 1.04 & \textminbf {0.40} & \textminbf {0.39} & \textminbf {0.43} \\ 
  {\ESTROWLASSO} TSA ERIC &  & 1.04 & 1.05 & 1.05 & \textminbf {0.51} & \textminbf {0.45} & \textminbf {0.38} & 1.03 & \textminbf {0.46} & 1.04 & 1.05 & 1.05 & 1.49 & \textminbf {0.46} & 1.02 & 1.03 & 1.02 \\ 
  {\ESTROWLASSO} TSA BIC &  & 1.02 & \textminbf {0.42} & 1.02 & 1.04 & 1.07 & \textminbf {0.38} & \textminbf {0.38} & \textminbf {0.46} & 1.02 & 1.02 & \textminbf {0.42} & \textminbf {0.51} & 1.07 & 1.02 & \textminbf {0.39} & \textminbf {0.43} \\ 
  {\ESTROWDANTZIG} TSA ERIC &  & 1.06 & 1.12 & 1.12 & 2.78 & 1.04 & 1.05 & 1.08 & 1.07 & 1.04 & 1.10 & 1.17 & 2.33 & 1.02 & 1.05 & 1.05 & 1.12 \\ 
  {\ESTROWDANTZIG} TSA BIC &  & 1.06 & 1.10 & 1.14 & 3.04 & 1.13 & 1.08 & 1.08 & 1.13 & 1.06 & 1.10 & 1.17 & 2.29 & 1.11 & 1.05 & 1.05 & 1.16 \\ 
   \hline
{\ESTVECLASSO} SA ERIC & \multirow{7}{0.5cm}{25} & 1.02 & 1.17 & 1.22 & 1.14 & 1.02 & 1.17 & 1.24 & 1.19 & 1.03 & 1.16 & 1.27 & 1.21 & 1.02 & 1.15 & 1.24 & 1.18 \\ 
  {\ESTVECLASSO} TSA ERIC &  & 1.03 & 1.15 & 1.22 & 1.16 & 1.03 & 1.17 & 1.22 & 1.21 & 1.05 & 1.16 & 1.27 & 1.23 & 1.05 & 1.15 & 1.24 & 1.20 \\ 
  {\ESTROWLASSO} SA BIC &  & \textminbf {0.61} & \textminbf {0.48} & \textminbf {0.45} & \textminbf {0.51} & \textminbf {0.61} & \textminbf {0.48} & 1.02 & 1.06 & \textminbf {0.60} & \textminbf {0.49} & 1.02 & 1.02 & \textminbf {0.62} & \textminbf {0.48} & \textminbf {0.45} & 1.06 \\ 
  {\ESTROWLASSO} TSA ERIC &  & 1.03 & 1.04 & 1.07 & 1.04 & 1.03 & 1.06 & 1.07 & 1.04 & 1.03 & 1.04 & 1.09 & 1.04 & 1.02 & 1.08 & 1.07 & 1.04 \\ 
  {\ESTROWLASSO} TSA BIC &  & \textminbf {0.61} & \textminbf {0.48} & \textminbf {0.45} & \textminbf {0.51} & \textminbf {0.61} & \textminbf {0.48} & \textminbf {0.45} & \textminbf {0.48} & 1.02 & \textminbf {0.49} & \textminbf {0.44} & \textminbf {0.48} & 1.02 & \textminbf {0.48} & \textminbf {0.45} & \textminbf {0.50} \\ 
  {\ESTROWDANTZIG} TSA ERIC &  & 1.05 & 1.08 & 1.18 & 1.90 & 1.05 & 1.10 & 1.13 & 2.25 & 1.05 & 1.10 & 1.16 & 4.90 & 1.03 & 1.10 & 1.16 & 1.86 \\ 
  {\ESTROWDANTZIG} TSA BIC &  & 1.07 & 1.10 & 1.44 & 1.04 & 1.07 & 1.13 & 1.20 & 2566.75 & 1.08 & 1.10 & 1.16 & 3.44 & 1.06 & 1.10 & 1.11 & 5.64 \\ 
   \hline
{\ESTVECLASSO} SA ERIC & \multirow{7}{0.5cm}{50} & 1.06 & 1.22 & 1.33 & 1.33 & 1.06 & 1.24 & 1.33 & 1.33 & 1.08 & 1.24 & 1.33 & 1.29 & 1.06 & 1.25 & 1.31 & 1.31 \\ 
  {\ESTVECLASSO} TSA ERIC &  & 1.09 & 1.24 & 1.33 & 1.35 & 1.12 & 1.25 & 1.35 & 1.33 & 1.11 & 1.24 & 1.33 & 1.29 & 1.11 & 1.25 & 1.33 & 1.33 \\ 
  {\ESTROWLASSO} SA BIC &  & \textminbf {0.67} & 1.02 & \textminbf {0.51} & \textminbf {0.49} & \textminbf {0.65} & 1.02 & \textminbf {0.51} & 1.02 & 1.02 & \textminbf {0.54} & \textminbf {0.51} & \textminbf {0.51} & \textminbf {0.65} & 1.02 & \textminbf {0.52} & \textminbf {0.51} \\ 
  {\ESTROWLASSO} TSA ERIC &  & 1.10 & 1.06 & 1.02 & 1.06 & 1.17 & 1.10 & 1.02 & 1.08 & 1.16 & 1.06 & 1.02 & 1.06 & 1.15 & 1.08 & \textminbf {0.52} & 1.06 \\ 
  {\ESTROWLASSO} TSA BIC &  & \textminbf {0.67} & \textminbf {0.54} & 1.02 & \textminbf {0.49} & \textminbf {0.65} & \textminbf {0.51} & 1.02 & \textminbf {0.49} & \textminbf {0.64} & \textminbf {0.54} & 1.02 & \textminbf {0.51} & \textminbf {0.65} & \textminbf {0.52} & \textminbf {0.52} & 1.02 \\ 
  {\ESTROWDANTZIG} TSA ERIC &  & 1.10 & 1.13 & 1.08 & 1.63 & 1.17 & 1.16 & 1.08 & 1.82 & 1.17 & 1.13 & 1.10 & 2.59 & 1.15 & 1.15 & 1.08 & 1.10 \\ 
  {\ESTROWDANTZIG} TSA BIC &  & 1.07 & 1.17 & 1.14 & 1.20 & 1.11 & 1.20 & 1.16 & 1.88 & 1.09 & 1.17 & 1.35 & 1.29 & 1.09 & 1.17 & 1.12 & 15.75 \\ 
   \hline
{\ESTVECLASSO} SA ERIC & \multirow{7}{0.5cm}{100} & 1.10 & 1.23 & 1.28 & 1.26 & 1.11 & 1.22 & 1.24 & 1.73 & 1.08 & 1.22 & 1.24 & 1.20 & 1.09 & 1.22 & 1.26 & 1.23 \\ 
  {\ESTVECLASSO} TSA ERIC &  & 1.12 & 1.20 & 1.27 & 1.26 & 1.11 & 1.20 & 1.24 & 1.73 & 1.11 & 1.20 & 1.24 & 1.20 & 1.11 & 1.20 & 1.25 & 1.21 \\ 
  {\ESTROWLASSO} SA BIC &  & \textminbf {0.73} & \textminbf {0.64} & \textminbf {0.60} & \textminbf {0.62} & \textminbf {0.74} & 1.02 & \textminbf {0.62} & \textminbf {0.63} & \textminbf {0.75} & \textminbf {0.65} & \textminbf {0.62} & \textminbf {0.64} & \textminbf {0.75} & 1.02 & \textminbf {0.61} & \textminbf {0.62} \\ 
  {\ESTROWLASSO} TSA ERIC &  & 1.70 & 1.95 & 1.62 & 1.35 & 1.72 & 1.95 & 1.58 & 1.33 & 1.65 & 1.85 & 1.56 & 1.42 & 1.64 & 1.95 & 1.59 & 1.53 \\ 
  {\ESTROWLASSO} TSA BIC &  & \textminbf {0.73} & \textminbf {0.64} & \textminbf {0.60} & 1.02 & \textminbf {0.74} & \textminbf {0.64} & \textminbf {0.62} & \textminbf {0.63} & \textminbf {0.75} & 1.02 & \textminbf {0.62} & \textminbf {0.64} & 1.01 & \textminbf {0.64} & \textminbf {0.61} & 1.02 \\ 
  {\ESTROWDANTZIG} TSA ERIC &  & 1.49 & 1.75 & 1.60 & 46.39 & 1.49 & 1.78 & 1.52 & 1.38 & 1.44 & 1.71 & 1.52 & 1.73 & 1.45 & 1.81 & 1.69 & 1.79 \\ 
  {\ESTROWDANTZIG} TSA BIC &  & 1.08 & 1.12 & 1.13 & 29.52 & 1.07 & 1.12 & 1.13 & 1.14 & 1.07 & 1.14 & 1.11 & 2.55 & 1.08 & 1.11 & 1.13 & 1.40 \\ 
   \hline
\end{tabular}
\caption{Example 2 -- VAR(1), $\int \|f(\omega)-\hat f(\omega)\|_\infty d\omega/\int \|f(\omega)\|_\infty d\omega$, n=100} 
\end{sidewaystable}
\begin{sidewaystable}[ht]
\centering
\begin{tabular}{|l|l|rrrr|rrrr|rrrr|rrrr|}
  \hline & $s$   & \multicolumn{4}{c|}{$1$}&\multicolumn{4}{c|}{$3$} &\multicolumn{4}{c|}{$5$} &\multicolumn{4}{c|}{$10$}    \\ \hline & \diagbox{$d$}{$\rho$}  & 0.6 & 0.8 & 0.9 & 0.95 & 0.6 & 0.8 & 0.9 & 0.95 & 0.6 & 0.8 & 0.9 & 0.95 & 0.6 & 0.8 & 0.9 & 0.95 \\ \hline 
 {\ESTVECLASSO} SA ERIC & \multirow{7}{0.5cm}{10} & \textminbf {1.06} & \textminbf {1.05} & \textminbf {1.06} & \textminbf {1.06} & 1.01 & \textminbf {1.09} & \textminbf {1.10} & \textminbf {1.11} & \textminbf {1.07} & \textminbf {1.07} & \textminbf {1.06} & \textminbf {1.06} & \textminbf {1.10} & \textminbf {1.11} & \textminbf {1.10} & \textminbf {1.11} \\ 
  {\ESTVECLASSO} TSA ERIC &  & 1.01 & 1.01 & \textminbf {1.06} & \textminbf {1.06} & 1.03 & 1.01 & 1.01 & 1.01 & 1.01 & 1.01 & \textminbf {1.06} & \textminbf {1.06} & 1.02 & 1.01 & 1.01 & \textminbf {1.11} \\ 
  {\ESTROWLASSO} SA BIC &  & 1.01 & \textminbf {1.05} & \textminbf {1.06} & \textminbf {1.06} & 1.02 & \textminbf {1.09} & \textminbf {1.10} & \textminbf {1.11} & 1.01 & \textminbf {1.07} & \textminbf {1.06} & \textminbf {1.06} & 1.02 & \textminbf {1.11} & \textminbf {1.10} & \textminbf {1.11} \\ 
  {\ESTROWLASSO} TSA ERIC &  & \textminbf {1.06} & \textminbf {1.05} & \textminbf {1.06} & \textminbf {1.06} & \textminbf {1.09} & \textminbf {1.09} & \textminbf {1.10} & 1.01 & \textminbf {1.07} & \textminbf {1.07} & \textminbf {1.06} & \textminbf {1.06} & \textminbf {1.10} & \textminbf {1.11} & \textminbf {1.10} & \textminbf {1.11} \\ 
  {\ESTROWLASSO} TSA BIC &  & 1.02 & \textminbf {1.05} & \textminbf {1.06} & \textminbf {1.06} & 1.03 & \textminbf {1.09} & \textminbf {1.10} & \textminbf {1.11} & 1.02 & \textminbf {1.07} & \textminbf {1.06} & \textminbf {1.06} & 1.03 & 1.01 & \textminbf {1.10} & \textminbf {1.11} \\ 
  {\ESTROWDANTZIG} TSA ERIC &  & \textminbf {1.06} & \textminbf {1.05} & \textminbf {1.06} & 1.01 & \textminbf {1.09} & \textminbf {1.09} & \textminbf {1.10} & \textminbf {1.11} & \textminbf {1.07} & \textminbf {1.07} & \textminbf {1.06} & \textminbf {1.06} & \textminbf {1.10} & \textminbf {1.11} & \textminbf {1.10} & \textminbf {1.11} \\ 
  {\ESTROWDANTZIG} TSA BIC &  & 1.02 & 1.01 & 1.01 & 1.01 & 1.03 & 1.01 & \textminbf {1.10} & \textminbf {1.11} & 1.02 & 1.01 & \textminbf {1.06} & 1.01 & 1.03 & 1.01 & 1.01 & \textminbf {1.11} \\ 
   \hline
{\ESTVECLASSO} SA ERIC & \multirow{7}{0.5cm}{25} & \textminbf {1.13} & 1.02 & 1.02 & 1.03 & \textminbf {1.13} & 1.02 & 1.02 & 1.02 & 1.01 & 1.02 & 1.03 & 1.03 & \textminbf {1.14} & 1.02 & 1.03 & 1.03 \\ 
  {\ESTVECLASSO} TSA ERIC &  & 1.04 & 1.03 & 1.03 & 1.04 & 1.04 & 1.03 & 1.03 & 1.03 & 1.04 & 1.03 & 1.03 & 1.03 & 1.04 & 1.03 & 1.03 & 1.03 \\ 
  {\ESTROWLASSO} SA BIC &  & 1.02 & \textminbf {1.11} & \textminbf {1.11} & \textminbf {1.10} & 1.02 & \textminbf {1.11} & \textminbf {1.10} & \textminbf {1.10} & 1.02 & 1.01 & \textminbf {1.10} & \textminbf {1.10} & 1.03 & \textminbf {1.11} & \textminbf {1.10} & \textminbf {1.10} \\ 
  {\ESTROWLASSO} TSA ERIC &  & \textminbf {1.13} & \textminbf {1.11} & \textminbf {1.11} & 1.01 & \textminbf {1.13} & \textminbf {1.11} & \textminbf {1.10} & \textminbf {1.10} & \textminbf {1.13} & \textminbf {1.13} & 1.01 & \textminbf {1.10} & \textminbf {1.14} & \textminbf {1.11} & 1.01 & \textminbf {1.10} \\ 
  {\ESTROWLASSO} TSA BIC &  & 1.04 & 1.01 & \textminbf {1.11} & \textminbf {1.10} & 1.04 & 1.01 & \textminbf {1.10} & \textminbf {1.10} & 1.04 & 1.02 & 1.01 & \textminbf {1.10} & 1.04 & 1.01 & 1.01 & \textminbf {1.10} \\ 
  {\ESTROWDANTZIG} TSA ERIC &  & 1.02 & 1.03 & 1.05 & 1.08 & 1.02 & 1.03 & 1.05 & 1.06 & 1.02 & 1.04 & 1.05 & 1.06 & 1.02 & 1.03 & 1.05 & 1.07 \\ 
  {\ESTROWDANTZIG} TSA BIC &  & 1.05 & 1.04 & 1.03 & 1.03 & 1.05 & 1.04 & 1.03 & 1.03 & 1.05 & 1.04 & 1.04 & 1.03 & 1.05 & 1.03 & 1.04 & 1.03 \\ 
   \hline
{\ESTVECLASSO} SA ERIC & \multirow{7}{0.5cm}{50} & 1.01 & 1.02 & 1.04 & 1.04 & \textminbf {1.17} & 1.02 & 1.04 & 1.03 & \textminbf {1.17} & 1.03 & 1.04 & 1.04 & 1.01 & 1.02 & 1.03 & 1.04 \\ 
  {\ESTVECLASSO} TSA ERIC &  & 1.06 & 1.06 & 1.06 & 1.05 & 1.05 & 1.05 & 1.06 & 1.04 & 1.05 & 1.05 & 1.06 & 1.05 & 1.06 & 1.04 & 1.06 & 1.06 \\ 
  {\ESTROWLASSO} SA BIC &  & 1.01 & 1.01 & 1.01 & \textminbf {1.13} & \textminbf {1.17} & \textminbf {1.14} & 1.01 & \textminbf {1.15} & \textminbf {1.17} & \textminbf {1.14} & 1.01 & \textminbf {1.15} & 1.01 & \textminbf {1.14} & 1.01 & \textminbf {1.15} \\ 
  {\ESTROWLASSO} TSA ERIC &  & \textminbf {1.18} & \textminbf {1.15} & \textminbf {1.15} & 1.01 & \textminbf {1.17} & \textminbf {1.14} & \textminbf {1.15} & \textminbf {1.15} & \textminbf {1.17} & \textminbf {1.14} & \textminbf {1.15} & \textminbf {1.15} & \textminbf {1.16} & \textminbf {1.14} & \textminbf {1.15} & \textminbf {1.15} \\ 
  {\ESTROWLASSO} TSA BIC &  & 1.03 & 1.03 & 1.03 & 1.01 & 1.02 & 1.02 & 1.02 & 1.01 & 1.02 & 1.02 & 1.03 & 1.01 & 1.03 & 1.02 & 1.03 & 1.02 \\ 
  {\ESTROWDANTZIG} TSA ERIC &  & \textminbf {1.18} & 1.02 & 1.03 & 1.04 & \textminbf {1.17} & 1.02 & 1.03 & 1.03 & \textminbf {1.17} & 1.02 & 1.03 & 1.04 & 1.01 & 1.02 & 1.03 & 1.04 \\ 
  {\ESTROWDANTZIG} TSA BIC &  & 1.05 & 1.05 & 1.06 & 1.05 & 1.04 & 1.04 & 1.06 & 1.05 & 1.04 & 1.05 & 1.06 & 1.05 & 1.05 & 1.04 & 1.06 & 1.06 \\ 
   \hline
{\ESTVECLASSO} SA ERIC & \multirow{7}{0.5cm}{100} & 1.01 & 1.04 & 1.05 & 1.06 & 1.02 & 1.03 & 1.05 & 1.07 & 1.02 & 1.03 & 1.05 & 1.05 & 1.02 & 1.03 & 1.05 & 1.06 \\ 
  {\ESTVECLASSO} TSA ERIC &  & 1.07 & 1.10 & 1.09 & 1.09 & 1.09 & 1.10 & 1.10 & 1.10 & 1.08 & 1.10 & 1.10 & 1.09 & 1.07 & 1.10 & 1.09 & 1.09 \\ 
  {\ESTROWLASSO} SA BIC &  & \textminbf {1.24} & \textminbf {1.20} & \textminbf {1.18} & \textminbf {1.16} & \textminbf {1.23} & \textminbf {1.21} & \textminbf {1.18} & \textminbf {1.17} & \textminbf {1.23} & \textminbf {1.23} & \textminbf {1.20} & \textminbf {1.17} & \textminbf {1.25} & \textminbf {1.21} & \textminbf {1.18} & \textminbf {1.16} \\ 
  {\ESTROWLASSO} TSA ERIC &  & 1.06 & 1.05 & 1.03 & 1.04 & 1.07 & 1.04 & 1.04 & 1.04 & 1.07 & 1.04 & 1.03 & 1.04 & 1.06 & 1.04 & 1.03 & 1.05 \\ 
  {\ESTROWLASSO} TSA BIC &  & 1.02 & 1.03 & 1.02 & 1.02 & 1.03 & 1.02 & 1.03 & 1.03 & 1.02 & 1.02 & 1.03 & 1.02 & 1.02 & 1.02 & 1.02 & 1.03 \\ 
  {\ESTROWDANTZIG} TSA ERIC &  & 1.04 & 1.04 & 1.05 & 1.07 & 1.05 & 1.04 & 1.05 & 1.07 & 1.05 & 1.03 & 1.05 & 1.07 & 1.04 & 1.03 & 1.04 & 1.07 \\ 
  {\ESTROWDANTZIG} TSA BIC &  & 1.05 & 1.06 & 1.06 & 1.07 & 1.06 & 1.06 & 1.06 & 1.07 & 1.06 & 1.06 & 1.06 & 1.06 & 1.05 & 1.05 & 1.05 & 1.07 \\ 
   \hline
\end{tabular}
\caption{Example 2 -- VAR(1), $1/d\sum_{j=1}^d MSE(\hat X_{n+1;j})/\sigma_j^2$, n=100} 
\end{sidewaystable}

\FloatBarrier
{\bf Acknowledgments.}  The research of the first author was supported by the Research Center (SFB) 884 ``Political Economy of Reforms''(Project B6), funded by the German Research Foundation
(DFG). Furthermore, the authors acknowledge support by the state of Baden-W{\"u}rttemberg through bwHPC.

\section{Proofs}
\begin{proof}[Proof of Theorem~\ref{thm.thres}]
Since for a matrix $A$ we have $\|A\|_\infty=\max_j \|e_j^\top A\|_1$ and $\|A\|_1=\max_j \|Ae_j\|_1$, it is sufficient to show that for all $j=1,\dots,p$,  $\|\A-\THRl(\hat \A)\|_\infty=\sum_{k=1}^d \|e_j^\top(A_k-\THRl(\hat A_k))\|_1 \leq C_2  t_n$ and $\|\A-\THRl(\hat \A)\|_1=\max_{1\leq k \leq p} \|(A_k-\THRl(\hat A_k))e_j\|_1 \leq C_2 t_n$. 
In order to bound $\sum_{k=1}^d \|e_j^\top(A_k-\THRl(\hat A_k))\|_1$ we can mainly follow the arguments used in  the proof of Theorem 1 in \cite{cai2011adaptive}. 
For some $j$ and by the conditions imposed on  the thresholding operation  and 
because  $(A_1,\dots,A_p)\in\mathcal{M}(q,s,M,p)$ and $\max_k \|A_k-\hat A_k\|_{\max}\leq {\lambdan}$, we have that 
\begin{align*}
    \sum_{k=1}^p \|e_j^\top(A_k-\THRl(\hat A_k))\|_1=&
    \sum_{k=1}^p \max_j \sum_{i=1}^d |A_{k;ji}-\THRl(\hat A_{k;ji})|\\
    =&\sum_{k=1}^p \max_j \sum_{i=1}^d |\THRl(\hat A_{k;ji})-A_{k;ji}|\ind(|\hat A_{k;ji}|>{\lambdan}, |A_{k;ji}|>{\lambdan})\\
    &+\sum_{k=1}^p \max_j \sum_{i=1}^d |A_{k;ji}|\ind(|\hat A_{k;ji}|\leq {\lambdan}, |A_{k;ji}|>{\lambdan})\\
    &+\sum_{k=1}^p\max_j \sum_{i=1}^d |\THRl(\hat A_{k;ji})-A_{k;ji}|\ind(|\hat A_{k;ji}|>{\lambdan}, |A_{k;ji}|\leq{\lambdan})\\
    &+\sum_{k=1}^p \max_j \sum_{i=1}^d |A_{k;ji}|^{q}|A_{k;ji}|^{1-q}\frac{\lambdan^{1-q}}{\lambdan^{1-q}}\ind(|\hat A_{k;ji}|\leq {\lambdan}, |A_{k;ji}|\leq{\lambdan})\\
    \leq& 2\sum_{k=1}^p \max_j \sum_{i=1}^d \lambdan^{q}\frac{|A_{k;ji}|^q}{|A_{k;ji}|^q}\lambdan^{1-q}\ind(|A_{k;ji}|>\lambdan)\\
    &+(1+c)\sum_{k=1}^p \max_j \sum_{i=1}^d |A_{k;ji}|\ind(|\hat A_{k;ji}|>{\lambdan}, |A_{k;ji}|\leq{\lambdan})\\
    &+\lambdan^{1-q} \sum_{k=1}^p \max_j\sum_{i=1}^d |A_{k;ji}|^{q}\ind(|\hat A_{k;ji}|\leq {\lambdan}, |A_{k;ji}|\leq{\lambdan})\\
    \leq& \lambdan^{1-q}s(4+c)=(4+c)C_1^{1-q}s( t_n)^{1-q}.
\end{align*}
This implies $\|\A- \THRl(\hat A)\|_\infty=\sum_{k=1}^p \|A_k-\THRl(\hat A_k)\|_\infty \leq (4+c)C_1^{1-q}s t_n^{1-q}$. $\|\A- \THRl(\hat A)\|_1=\max_{1\leq k \leq p} \|A_k-\THRl(\hat A_k)\|_1$ can be bounded by the same arguments.
\end{proof}

\begin{proof}[Proof of Corollary~\ref{cor.Dantzig}]
We have $(\hat A_1^\TC,\dots,\hat A_p^\TC)=\THRl(\hat \beta_1^\CL,\dots,\hat \beta_d^\CL)^\top$. For each $\hat\beta_j^\CL$, $j=1,\dots,d$,  an error bound with respect to the $\|\cdot\|_{\max}$ norm  is obtained by Theorem~4 in \cite{wu2016performance} on a set with probability of at least $\tilde p_n^{\CL}$ as defined in \eqref{eq.error.bound.wu}. Since all $\hat\beta_j^\CL$ share the same regressors, only the event denoted by $B$ in the proof of Theorem~4 in \cite{wu2016performance} differs among the $\hat\beta_j^\CL$. Hence, the probability of the intersection  of all these $p$ events, where the corresponding set is denoted  by $B$, and of the event denoted by $A$ in the proof of Theorem~4 in \cite{wu2016performance},  is  at least $p_n^\CL$.  Assertion \eqref{eq.Gammas.max.error.bound} and \eqref{eq.eps.max.error.bound} follow directly by arguments used in the proof of Theorem~4 in \cite{wu2016performance} with 
\[ a=( \frac{\sqrt{\log(d)}}{\sqrt N}+\frac{d^{4/\tau}}{N^{1-2/\tau}})(\sum_{j=0}^\infty \|\A^j\|_2 C_{\eps,\tau})^2\]
and 
\[ b=( \frac{\sqrt{\log(d)}}{\sqrt N}+\frac{d^{1/\tau}}{N^{1-1/\tau}})(\sum_{j=0}^\infty \|\A^j\|_2 C_{\eps,\tau})^2.\]
See also Example~1 and 4, and Remark 6 in \cite{wu2016performance}. 

On a set with probability of at least $p_n^\CL$, Theorem~4 in \cite{wu2016performance} leads  with  the above choice of $a$ and $b$, to  the bound 
\begin{align*}
\|B-\hat B^\CL\|_{\max} & \leq 2 \|\Gammas^{-1}\|_1 (\sum_{j=0}^\infty \|\A^j\|_2 C_{\eps,\tau})^2\\
& \ \ \ \times \Big[( \frac{\sqrt{\log(dp)}}{\sqrt N}+\frac{dp^{4/\tau}}{N^{1-2/\tau}})M+( \frac{\sqrt{\log(dp)}}{\sqrt N}+\frac{dp^{1/\tau}}{N^{1-1/\tau}})\Big]
\end{align*}
Following the proof of Theorem~6 in \cite{cai2011constrained}, i.e. the arguments leading to equation (27) in the aforecited paper, we obtain $\|B-\hat B^\CL\|_{\infty}\leq (1+2^{1-q}+3^{1-q}) s (2 \|\Gammas^{-1}\|_1 (\sum_{j=0}^\infty \|\A^j\|_2 C_{\eps,\tau})^2 \Big[( \frac{\sqrt{\log(dp)}}{\sqrt N}+\frac{dp^{4/\tau}}{N^{1-2/\tau}})M+( \frac{\sqrt{\log(dp)}}{\sqrt N}+\frac{dp^{1/\tau}}{N^{1-1/\tau}})\Big])^{1-q}$. Furthermore, we have 
\begin{align*}
&\|B-\hat B^\CL\|_{\max}\leq \| \Gammas^{-1} (\Gammas+\mathcal{X}^\top \mathcal{X}/N-\mathcal{X}^\top \mathcal{X}/N)(B-\hat B^\CL)\|_{\max} \\
&\leq \|\Gammas^{-1}\|_1 \Big(\|\Gammas-\mathcal{X}^\top \mathcal{X}/N\|_{\max} \|B-\hat B^\CL\|_\infty \\
& \ \ \ \ +\|\mathcal{X}^\top(\mathcal{Y}-\hat B^\CL)/N\|_{\max}+\|\mathcal{X}^\top \mathcal{\eps}/N\|_{\max}\Big),
\end{align*}
which gives expression  \eqref{eq.Dantzig.max}.  \eqref{eq.VAR.TC.error.bound} follows then directly by Theorem~\ref{thm.thres}.
\end{proof}

\begin{proof}[Proof of Theorem~\ref{thm.acf}]
Notice that  $\Gamma(h)^\st=\A^h \Gammas$ which implies $\hat \Gamma(h)^\st-\Gamma(h)^\st=\hat \A^h(\Gammah-\Gammas)+(\hat \A^h-\A^h)\Gammas.$ Furthermore, we have $\hat\Gamma(0)^\st-\Gamma(0)^\st=\sum_{s=1}^\infty (\hat \A^{s}-\A^{s}) \Sigma_U (\A^s)^\top+\sum_{s=0}^\infty \hat \A^{s} (\hat \Sigma_U-\Sigma_U) (\A^s)^\top+\sum_{s=1}^\infty \hat \A^{s} \hat\Sigma_U (\hat A^s-\A^s)^\top. 
$ Following the proof of Lemma 8 in \cite{krampe2018bootstrap} we have for $s\geq 1$, $\hat \A^s -\A^s=\sum_{j=0}^{s-1}\hat \A^j(\hat \A-\A)\A^{s-1-j}.$
Hence, 
$\hat\Gamma(0)^\st-\Gamma(0)^\st=\sum_{s=1}^\infty  \sum_{j=0}^{s-1}\hat \A^j(\hat \A-\A)\A^{s-1-j} \Sigma_U (\A^s)^\top+\sum_{s=0}^\infty \hat \A^{s} (\hat \Sigma_U-\Sigma_U) (\A^s)^\top+\sum_{s=1}^\infty \hat \A^{s} \hat\Sigma_U (\sum_{j=0}^{s-1}[\hat \A]^j(\hat \A-\A)\A^{s-1-j})^\top$. For some sub-multiplicative matrix norm $\|\cdot\|$ and since $ab\leq (a^2+b^2)/2$, we further have
\begin{align*}
\|\hat\Gamma(0)^\st-\Gamma(0)^\st\|\leq& \|\hat \A -\A\| 
\|\Sigmaeps\| \sum_{j=0}^\infty \|\hat \A^j \| \|(\A^j)^\top\| \sum_{s=0}^\infty \|\A^s \| \|(\A^s)^\top\|\\&
+ \|\Sigmaeps-\Sigmah\| \sum_{j=0}^\infty \|\hat \A^j \| \|(\A^j)^\top\|\\&+\|\hat \A^\top -\A^\top\| \|\Sigmah\| \sum_{j=0}^\infty \|\hat \A^j \| \|(\A^j)^\top\| \sum_{s=0}^\infty \|\hat \A^s \| \|(\hat \A^s)^\top\| \\
\leq& \|\hat \A -\A\| C_{\gamma,\Sigmaeps} (C_{\gamma,\hat A}+C_{\gamma, A^\top})(C_{\gamma, A}+C_{\gamma,A^\top})/4 +
\|\Sigmah -\Sigmaeps\| (C_{\gamma,\hat A}+C_{\gamma, A^\top})/2 \\
&+ \|\hat \A^\top -\A^\top\| (C_{\gamma,\Sigmaeps}+\|\Sigmah -\Sigmaeps\|) (C_{\gamma,\hat A}+C_{\gamma, A^\top})(C_{\gamma, \hat A}+C_{\gamma,\hat A^\top})/4.
\end{align*}
\end{proof}

\begin{proof}[Proof of Theorem ~\ref{thm.spec}]
We first write  
\begin{align*}
f^{-1}(\omega)-\hat f^{-1}(\omega) & =
(\mathcal{\hat A}(\exp(i\omega))-\mathcal{\hat A}(\exp(i\omega)))^\top \Sigmaeps^{-1}  \mathcal{A}(\exp(-i\omega))\\
& \ \ +
\mathcal{A}(\exp(i\omega))^\top (\Sigmah^{-1}-\Sigmaeps^{-1} )\mathcal{A}(\exp(-i\omega))\\
& \ \ +
\mathcal{A}(\exp(i\omega))^\top \Sigmaeps^{-1}  (\mathcal{\hat A}(\exp(-i\omega))-\mathcal{A}(\exp(-i\omega)))\\
& \ \ +
(\mathcal{\hat A}(\exp(i\omega))-\mathcal{ A}(\exp(i\omega)))^\top (\Sigmah^{-1}-\Sigmaeps^{-1} )\mathcal{A}(\exp(-i\omega)) \\
& \ \ +
\mathcal{A}(\exp(i\omega))^\top (\Sigmah^{-1}-\Sigmaeps^{-1} )(\mathcal{\hat A}(\exp(-i\omega))-\mathcal{ A}(\exp(-i\omega))) \\
& +
(\mathcal{\hat A}(\exp(i\omega))-\mathcal{ A}(\exp(i\omega)))^\top (\Sigmah^{-1}-\Sigmaeps^{-1} )(\mathcal{\hat A}(\exp(-i\omega))-\mathcal{ A}(\exp(-i\omega))).
\end{align*}
Observe that  
\begin{align*}
\|(\mathcal{\hat A}( \exp(i\omega))-\mathcal{\hat A}(\exp(i\omega)))^\top\|_1 & \leq \sum_{s=1}^p \|\hat A_s^\top-A_s^\top\|_1\\
& =\sum_{s=1}^p \|\hat A_s-A_s\|_\infty\leq t_{n,1}
\end{align*}
and that
\begin{align*}
\|(\mathcal{\hat A}(\exp(i\omega))-\mathcal{\hat A}(\exp(i\omega)))^\top\|_\infty\leq \sum_{s=1}^p \|\hat A_s-A_s\|_1\leq t_{n,1}.
\end{align*}
Hence, $\|(\mathcal{\hat A}(\exp(i\omega))-\mathcal{\hat A}(\exp(i\omega)))^\top\|_l\leq t_{n,1}$ and $\|\mathcal{\hat A}(\exp(i\omega))-\mathcal{\hat A}(\exp(i\omega))\|_l\leq t_{n,1}$ for all $l \in [1,\infty]$. Since $\|\cdot\|_l$ is sub-multiplicative and $(A_1,\dots,A_p)\in \mathcal{M}(q,s,M,p)$, $\Sigmaeps^{-1}\in \mathcal{M}(q_{\eps^{-1}},s_{\eps^{-1}},M_{\eps^{-1}},1)$, the assertion follows. 
\end{proof}

\bibliographystyle{apalike}
\bibliography{bib}




\end{document}